\documentclass{article}

\usepackage[usenames,dvipsnames,svgnames]{xcolor}

\usepackage{hyperref}
\hypersetup{
    colorlinks=true,
    citecolor=MidnightBlue,
    urlcolor=Bittersweet,
}

\usepackage[margin=1in]{geometry}

\usepackage{float}
\usepackage{bbm}
\usepackage{bm}
\usepackage[normalem]{ulem}
\usepackage{extarrows}
\usepackage{enumitem}
\usepackage{comment}
\usepackage{amsthm,amsmath,amsfonts}
\usepackage{booktabs}
\usepackage{algorithm,algpseudocode} 

\usepackage{cleveref}
\newtheorem{theorem}{Theorem}
\newtheorem{lemma}[theorem]{Lemma}
\newtheorem{corollary}[theorem]{Corollary}
\newtheorem{proposition}[theorem]{Proposition}
\newtheorem{remark}[theorem]{Remark}

\usepackage[compact]{titlesec}

\usepackage{graphicx}
\usepackage{bbm}
\newcommand{\R}{\mathbb{R}}
\renewcommand{\P}{\mathbb{P}}

\newcommand{\E}{\mathbb{E}\,}
\newcommand{\X}{\mathbf{X}}
\newcommand{\Y}{\mathbf{Y}}

\newcommand{\e}{\mathrm{e}}

\usepackage{tikz}
\newcommand*\circled[1]{\tikz[baseline=(char.base)]{
   \node[shape=circle,draw,inner sep=1pt] (char) {\footnotesize #1};}}
   
\newcommand{\autocorr}{R} 
\newcommand{\Autocorr}{\mathbf{R}} 
\newcommand{\transpose}{\top}
\newcommand{\trans}{\transpose}
\newcommand{\covmat}{\boldsymbol{\Sigma}}

\usepackage{xspace}

\newcommand{\sketch}{\boldsymbol{\Omega}}
\newcommand{\be}{\mathbf{e}}
\newcommand{\bx}{\mathbf{x}}
\newcommand{\bv}{\mathbf{v}} 
\newcommand{\bw}{\mathbf{w}}
\newcommand{\bH}{\mathbf{H}}
\newcommand{\bD}{\mathbf{D}}
\newcommand{\bP}{\mathbf{P}}
\newcommand{\eps}{\varepsilon}
\renewcommand{\epsilon}{\varepsilon}
\newcommand{\longT}{T_\text{long}}
\newcommand{\Var}{\mathbb{V}\text{ar}}

\makeatletter
\newcommand{\del}[1]{
  \@bsphack
  \@esphack
}
\makeatother
\let\remove\del           
\let\rem\del              
\newcommand{\new}[1]{#1}  
\let\add\new              
\newenvironment{newstuff}{\ignorespaces}{\ignorespacesafterend}

\begin{document}
\title{Spectral estimation from simulations via sketching\footnote{Published in \emph{J.\ Comp. Phys.} Dec 2021, \href{https://doi.org/10.1016/j.jcp.2021.110686}{doi.org/10.1016/j.jcp.2021.110686}}}

\author{Zhishen Huang and Stephen Becker\footnote{Department of Applied Mathematics, University of Colorado Boulder, Boulder, CO, 80309}}

\maketitle

\begin{abstract} 
Sketching is a stochastic dimension reduction method that preserves geometric structures of data and has applications in high-dimensional regression, low rank approximation and graph sparsification. In this
work, we show that sketching can be used to compress simulation data and still  accurately estimate time autocorrelation and power spectral density. For a given compression ratio, the accuracy is much higher than using previously known methods.
In addition to providing theoretical guarantees, 
we apply sketching to a molecular dynamics simulation of methanol and find that
the estimate of spectral density is 90\% accurate using only 10\% of the data.
\end{abstract}

Large-scale computer simulations are a common tool in many disciplines like 
astrophysics, cosmology, fluid dynamics, computational chemistry, meteorology and oceanography, to name just a few.
 In many of these fields, a key goal of the simulation is an estimate of the power spectral density (or equivalently autocorrelation) of some dynamic or thermodynamic state variable or derived function. 
 
Computing a full autocorrelation becomes prohibitively expensive for large-scale simulations since it requires storing the entire dataset in memory.
 The textbook strategy to combat this problem is to subsample in time, often with clever logarithmic or multi-level spacing strategies~\cite{FRENKEL_book}.  Other simple solutions subsample particles or grid points, or both time and particles/points. Unfortunately, these {\it ad hoc} methods lack rigorous performance guarantees and can have arbitrarily large error.  
 This article shows how to leverage results from the new field of {\it randomized linear algebra} to derive subsampling methods that work better in practice and have theoretical guarantees on the accuracy.
 These new subsampling methods, known as {\it sketching} methods, essentially exploit the fact that 
 \del{when}
 multiplying by a multivariate Gaussian to do compression 
 \new{ensures}
 \del{there are}
 no worst-case inputs; in comparison, simple subsampling methods do well on some inputs but catastrophically bad on other inputs.  Section 
\ref{sec:sketching} gives a toy example of this, and the rest of the paper shows how this applies to sampling data for spectral estimation. 

\begin{newstuff}
\paragraph{Contributions} This paper shows how to use existing results from  randomized linear algebra results in the context of estimating autocorrelations and power spectral densities. Specifically, we
\begin{enumerate}
    \item show that the autocorrelation and power spectral density are simple functions of the covariance matrix; 
    \item convert existing results on covariance matrix estimation to results on estimating autocorrelation and power spectral density; and 
    \item numerically demonstrate that the resulting sketching methods are significantly more accurate than baseline methods when applied to the problem of autocorrelation and power spectral density estimation in a typical molecular dynamic simulation.
\end{enumerate}
\end{newstuff}

Throughout the paper, we pay attention to computation and communication costs. In particular, the sketches are linear operators and can be applied to a data stream, so they can be applied during a simulation with negligible memory overhead and in a reasonable time. Our methods are also simple to implement. Indeed, a reason that more sophisticated sampling schemes are not used in practice may be due to the cumbersome book-keeping required for normalizations, but we review a simple trick to deal with this (Remark~\ref{rmk:normalization}), and other than sampling, our methods do not require any ``on-the-fly'' computation, as the estimates are formed in post-processing.

\paragraph{Background}
Spectral estimation arises  
 in molecular dynamic (MD) simulations based on time-dependent density functional theory (TDDFT) \cite{TDDFT1984}, which is a prominent methodology for electronic structure calculations. 
Depending on the original variable (position, velocity, dipole-moment, etc.), 
applications of spectral estimation in TDDFT
include calculating 
 vibrational or rotational modes (as used in infrared and Raman spectroscopy)~\cite{scott1996harmonic}, optical absorption spectra~\cite{OpticalAbsorption96},
and circular dichroism spectra~\cite{CircularDichroism09}.
Many of these quantities can be experimentally measured, so \del{one use of} the spectrum \add{can be used} \del{is} to verify that the simulation matches with reality, \add{as well as predicting} \del{or to predict} properties of novel materials.

Similarly, temporal autocorrelations may be computed during numerical solutions of partial differential equations (PDEs). For one example, in fluid dynamics, the autocorrelations computed via direct numerical simulation 
of the Navier-Stokes equations can be used to validate large-eddy simulation 
models~\cite{DNS_autocorr}. Another example is oceanography where modern simulation codes rely on multi-scale numerical methods that cannot fully resolve the smallest scales, and so use stochastic models to inform the simulation~\cite{grooms2013efficient,grooms2019diagnosing}. The stochastic process can be constrained to conform to a given autocorrelation function.

MD simulations operate on particles, while standard numerical methods for PDEs operate on (possibly unstructured) grids and elements. In both cases, the exact sample time-autocorrelation function can be computed provided the data (particles or grid points, at all times) is stored. Due to advances in computing power and algorithm design, it is now feasible to run extremely large simulations. A consequence of this is that many large-scale simulations generate more data than can be stored. 
As an example, running the billion-atom Lennard Jones benchmark on the MD \texttt{LAMMPS} software~\cite{LAMMPS} for the equivalent of 1 ns of simulation time on argon atoms~\cite{rapaport2004art} takes 4.9 hours on a 288 node GPU computer from 2012~\cite{LAMMPS_billion}, making it a modest large-scale computation. Storing the 6 coordinates of position and velocity in double precision for the $10^5$ timesteps would require 4.26 PB, well beyond a typical high-end cluster disk quota of 150 TB. Longer simulations, or simulations of molecules, only exacerbate the problem. Standard compression methods for scientific data, like \texttt{fpzip}~\cite{lindstrom2006fast} and \texttt{ZFP}~\cite{lindstrom2014fixed}, improve this by one or two orders of magnitude at best~\cite{salloum2018optimal}.

\section{Sketching}  \label{sec:sketching}

Sketching is used to reduce dimensionality from $N$ dimensions to some $m \ll N$. 
A family of sketches is a probability distribution on the set of real or complex $m \times N$ matrices such that if $\sketch$ is drawn from this family,  for any fixed vectors $\bv,\bw\in \R^N$, then $\|\sketch \bv - \sketch \bw\|_2 \approx \|\bv - \bw \|_2$ with high probability. Hence the sketch preserves distances, and by the  polarization formula, preserves inner products as well.
The core ideas behind sketching have been in place since the 1980s, and were well-known in theoretical computer science literature, but the field has expanded since 2005 as many applications in scientific computing were developed. In particular, sketching is often used to efficiently find solutions of large least-square regression problems~\cite{Clarkson_leastsqrt_13,Clarkson_l1regression_2005,Meng_Mahoney_13,Sohler_Woodruff_2011,Woodruff_Zhang_2013,Clarkson2012f}, and to determine the row and column space of large matrices for low-rank matrix decomposition \cite{halko2011finding,Drineas_CUR_08,Mahoney_Drineas_PNAS}.

Formally, a probability distribution on $m \times N$ matrices is a {\it Johnson-Lindenstrauss Transform} with parameters $\varepsilon,\delta$ and $d$ if for any fixed set of $d$ vectors $\{\bv_i\}_{i=1}^d \subset \R^N$, if $\sketch$ is drawn from this distribution, then with probability at least $1-\delta$ it holds that 
\[
(1-\eps)\|\bv_i - \bv_j\|_2^2 \le \|\sketch\bv_i - \sketch\bv_j\|_2^2 \le (1+\eps)\|\bv_i - \bv_j\|_2^2
\]
for all $i,j\in\{1,\ldots,d\}$. 
When no confusion arises, it is common to not distinguish between the random variable and the distribution, and write $\sketch \in \text{JLT}(\eps,\delta,d)$ to encode the notion.  The name Johnson-Lindenstrauss Transform honors Johnson and Lindenstrauss' well-known result which shows that such distributions exist for $m = \mathcal{O}(\epsilon^{-2}\log(d))$~\cite{johnson1984extensions}.

\paragraph{Intuition}
\del{The classic example of a sketch is an appropriately scaled Gaussian matrix with independent entries.} To gain insight, consider the case when $\sketch \in \R^{1 \times N}$ is a sketch that compresses $\bv\in \R^N$ to a single number, and without loss of generality, let $\|\bv\|_2=1$. All sketches we consider will be unbiased, meaning $\E \sketch^T\sketch = I_{N \times N}$ where $I$ is the identity matrix. We wish to preserve norm, so we look at $\| \sketch \bv \|_2^2$, or equivalently $(\sketch \bv)^2$ when $m=1$. Then any unbiased sketch has $\E (\sketch \bv)^2 = 1$. 

\new{A natural approach to reducing dimension is simple subsampling, meaning that each entry has an equal chance of being selected.}
Simple subsampling can be written as a sketch by defining $\sketch = \sqrt{N} \be_i^\trans$ where $\be_i$ is the $i^\text{th}$ canonical basis vector in $\R^N$, and $i$ is chosen uniformly from $\{1,\ldots,N\}$; one can easily show this is unbiased. \new{In the lucky event that} the input $\bv$ 
has weight evenly distributed over all coordinates, such that $|v_j|=N^{-1/2}$ for all $j=1,\ldots,N$, then this is a good sketch, since the variance is $\Var( (\sketch \bv )^2 ) = 0$. However, if the input is $\bv = \be_k$ for any fixed $k$, then an elementary calculation shows that $\Var( (\sketch \bv )^2 ) = N-1$, which in high dimensions is too large to be useful. 

In contrast, \new{the classic example of a \emph{good} sketch is an appropriately scaled Gaussian matrix with independent entries. For this sketch, 
define} $\sketch$ as $1 \times N$ independent standard normal random variables, then $\sketch$ is also an unbiased sketch, and furthermore $\Var( (\sketch \bv )^2 ) = 2$ independent of the fixed vector $\bv$. \new{In contrast, the variance of the simple subsampling sketch ranges between $[0,N-1]$ depending on $\bv$}.
The Gaussian sketch is not always more efficient than the subsampling sketch, but it is never much worse, and sometimes it is better by a factor of $N$.

\paragraph{Types of sketches}
In this work we consider the following three types of distributions of sketching matrices $\sketch$ (Matlab code available via \cite{StephensCode};
some Python implementations are part of the \texttt{random\_projection} module of scikit learn):
\begin{description}
    \item[Gaussian sketch]  Each entry of $\sketch$ is independently drawn from 
    the scaled normal distribution $\mathcal{N}(0,\frac{1}{m})$. 
    
    \item[Haar sketch] Draw $\widetilde{\sketch}$ as in the Gaussian case and then define the rows of $\sketch$ to be the output of Gram-Schmidt orthogonalization applied to the rows of $\widetilde{\sketch}$, scaled by $\sqrt{\frac{N}{m}}$. This is equivalent to sampling the first $m$ columns of a matrix from the Haar distribution on orthogonal matrices, and can also be computed via the \texttt{QR} factorization algorithm with post-processing~\cite{HaarMezzadri}. This is essentially the case originally considered by Johnson and Lindenstrauss.
    
    \item[FJLT] The Fast Johnson-Lindenstrauss Transformation (FJLT) as is usually implemented~\cite{Woodruff_Now} is a structured matrix of the form $\sketch = \sqrt{\frac{N}{m}}\bP^\trans \bH \bD$
    where $\bD$ is a diagonal matrix with Rademacher random variables on the diagonal (i.e., independent, $\pm 1$ with equal probablity), $\bH$ is a unitary or orthogonal matrix, and $\bP^\trans$ a simple subsampling matrix such that $\bP^\trans \bv$ chooses $m$ of the coordinates from $\bv$ uniformly at random (with replacement), so that $\bP$ consists of $m$ canonical basis vectors. To be useful, each entry of $\bH$ should be as small as possible ($\approx 1/\sqrt{N}$), and $\bH$ should be computationally fast to apply to vector. Standard choices for $\bH$ are the (Walsh-)Hadamard, discrete Fourier, and discrete Cosine transforms, all of which have fast implementations that take $\mathcal{O}(N \log N)$ flops to apply to a vector. Since applying $\bD$ and $\bP^\trans$ take linear and sub-linear time, respectively, the cost of computing $\sketch\bv$ is $\mathcal{O}(N\log N)$, better than the $\mathcal{O}(Nm)$ cost of the Gaussian and Haar sketches. 
    The original FJLT proposed in \cite{FastJL} is a slight variant that uses a different sparse matrix $\bP$.
    
\end{description}

There are other types of sketches such as the count-sketch~\cite{Cormode_sketch_techniques}, leverage-score based sketches \cite{MahoneyMonograph}, and entry-wise sampling~\cite{AchlioptasMcSherry,AKL2013Sampling} which can be combined with preconditioning~\cite{FarhadPrecond}. Some of these sketches are not Johnson-Lindenstrauss transforms but are instead the related notion of subspace embeddings.  See \cite{Woodruff_Now,MahoneyMonograph,martinsson2020randomized} for surveys on sketching literature.

\paragraph{Guarantees}

Table~\ref{table::JLT} summarizes the required compressed dimension size $m$ for the corresponding sketching matrix to be a JLT($\varepsilon,\delta,d$). 
\begin{table}[H]
\centering
\begin{tabular}{ lc } 
\toprule
Method & Compressed dimension $m$\\
\midrule
Gaussian \cite{Woodruff_Now} & $\mathcal{O}(\varepsilon^{-2}\log (d/\delta))$ \\
Haar \cite{vershynin_book}  & $\mathcal{O}(\varepsilon^{-2}\log (d/\delta))$ \\
FJLT \new{\cite{RandLA_ModelReduction2019},[Prop.~3.9]}& 
\new{$\mathcal{O}\left(
\varepsilon^{-2}\log\left(Nd/\delta\right)\log\left(d/\delta\right) \right)$}\\
\bottomrule
\end{tabular}
\caption{Compressed dimension requirement for JLTs.}
\label{table::JLT}
\end{table}

\begin{newstuff}
The result for the FJLT holds when $\bH$ is a Hadamard matrix, and follows from the observation that a subspace embedding with complexity that depends only logarithmically on the failure probability $\delta$ can be turned into a JLT using the union bound. When $\bH$ is a discrete Fourier or discrete Cosine transform, similar $\mathcal{O}(\varepsilon^{-2})$ sample complexities hold (with polylog factors in $d$, $N$ and $\delta^{-1}$) by combining \cite[Thm.~3.1]{RIP_2_JLT} with \cite[Thm.~12.31]{FoucartRauhut_CS_Book}.
\end{newstuff}
\del{
The result for the FJLT, which holds when $\bH$ is a  Hadamard, discrete Fourier or discrete Cosine transform, is not explicitly in the literature but follows by combining \cite[Thm.~3.1]{RIP_2_JLT} with \cite[Thm.~12.31]{FoucartRauhut_CS_Book}.}
The constants hidden in the asymptotic notation are not bad. For example, for the Gaussian sketch, with $d=10^3$ points (in arbitrary dimension \new{$N$}), for failure probability $\delta \le 0.1$ and error $\eps \le 1/3$, the number of samples required is $m\ge 535$. 

\section{Approximating Autocorrelation with Sketching} \label{sec:autocorr}
\newcommand{\signal}{x}
\newcommand{\spc}{\varphi} 
Throughout the article, we think of the data as a signal $\signal(t,\spc)$ in time $t$ and space $\spc$\xspace, where $\spc$ can encode a grid location or a particle number depending on the type of simulation (for space indices in dimension greater than one, we flatten the indices into a large one-dimensional list). Let $t$ have unit spacing $\Delta T=1$, $t\in\{1,2,\ldots,T\}$, and let space be indexed by $\{\spc_1,\ldots,\spc_N\}$. We organize the data into a matrix $\X\in\R^{T\times N}$.

In what follows, we consider classical methods for estimating the autocorrelation.  There are powerful alternative methods, based on parametric models --- most notably, autoregressive-moving-average (ARMA) models~\cite{broersen2006automatic}. However, these methods excel when $T$ is small, \rem{and} do not clearly extend to $N>1$, and are not natively suited to on-the-fly calculations during a simulation as they require significant post-processing and parameter tuning.

\paragraph{Autocorrelation and the Wiener-Khinchin Theorem}

For a continuous signal $\signal$, the time autocorrelation function of lag $\tau$ of signal $\signal$ is 
\[
R(\tau) = \E_{\spc} \lim_{T\to\infty}\frac{1}{2T}\int_{-T}^{T} \signal(t,\spc)\signal(t+\tau,\spc) \,\mathrm{d}t.
\]
For the corresponding discretized signal of length $T$, the (sample) time autocorrelation of lag $\tau$ is defined as  
\begin{align}
    \label{eqn::autocorr_defn}
    \widehat{\autocorr}_\tau[{\X}] 
    &= \frac{1}{N}\frac{1}{T-\tau} \sum_{t=1}^{T-\tau} \sum_{i=1}^{N} \signal(t,\spc_i) \signal(t+\tau,\spc_i)
\end{align}
\new{where we change notation slightly to emphasize that this is a function of the data $\X$.}
As our goal will be to approximate the sample autocorrelation $\widehat{\autocorr}_\tau$, we drop the $\widehat{\phantom{\autocorr}}$ notation for clarity and simply write $\autocorr_\tau$.

\begin{remark}[Cross-terms]
Calculating Eq.~\ref{eqn::autocorr_defn} requires storing $N\times T$ parameters.  If one instead computed $ \sum_{t=1}^{T-\tau} \left(\sum_{i=1}^{N} \signal(t,\spc_i)\right)\left( \sum_{i=1}^{N}\signal(t+\tau,\spc_i)\right)$ (with appropriate normalization), then only $\mathcal{O}(T)$ storage is required, but unfortunately this is not equivalent to Eq.~\ref{eqn::autocorr_defn} due to the presence of the cross-terms.  One way to view sketching methods is that the sketching adds in suitable randomness so that when using the $\mathcal{O}(T)$ formula, the cross-terms vanish in expectation.
\end{remark}

Letting the shifted, unnormalized (sample) covariance matrix be $\covmat = \X\X^\trans$, our first observation is that $\autocorr_\tau$ is a linear function of $\covmat$, since
\[
(\covmat)_{t,t'} = \sum_{i=1}^N \signal(t,\spc_i)\signal(t',\spc_i)
\]
so $\autocorr_\tau$ is the scaled sum of the $\tau^\text{th}$ diagonal of $\covmat$,
and hence we use the notation $\autocorr_\tau[\covmat]$, and also 
write
$\Autocorr[\covmat] = ({\autocorr}_0[\covmat],{\autocorr}_1[\covmat],\cdots,{\autocorr}_{T-1}[\covmat])^\trans$ when working with all $T$ possible lags.

The time autocorrelation is often of interest itself, but it can also be used to derive the power spectral density,
\[
S(\omega) = \lim_{T\to\infty}\E_\spc \left|\frac{1}{\sqrt{2T}}\int_{-T}^T \signal(t,\spc) \e^{-\mathrm{i}\omega t} \,\mathrm{d}t\right|^2.
\]
If $\signal$ is a wide-sense stationary random process, under certain conditions, the Wiener-Khinchin Theorem states that the spectral density is the Fourier transform of $R(\tau)$, and the discrete power spectral density can be estimated by the discrete Fourier transform of $\Autocorr$.

Thus both autocorrelation and power spectrum can be reduced to the problem of finding an accurate estimate of $\covmat$. Note that $\covmat$ is a $T\times T$ matrix \del{and} \new{that is} impractical to store, and is used only for analysis. Our actual software implementation only needs a factored form  $\covmat = \widehat{\X}\widehat{\X}^\trans$ for $\widehat{\X}\in\R^{T\times m}$, and works directly  with $\widehat{\X}$. Furthermore, due to linearity, implementations can exploit existing autocorrelation software (which typically use the fast Fourier transform to do convolutions efficiently).
Specifically, if the columns of $\widehat{\X}$ are $\bv_1,\ldots,\bv_m$, then $\autocorr_\tau[\covmat]=\autocorr_\tau[\sum_{i=1}^m \bv_i\bv_i^\trans] = \sum_{i=1}^m \autocorr_\tau[\bv_i\bv_i^\trans]$
and $\autocorr_\tau[\bv_i\bv_i^\trans]$ is performed implicitly via an efficient autocorrelation implementation.

In the next section, we \del{will} use standard results from the sketching literature to create an estimator $\widehat{\covmat}$ and bound $\|\covmat - \widehat{\covmat}\|_F<\eps$, where $\|\cdot\|_F$ denotes the Frobenius (Hilbert-Schmidt) norm. To use those results, we first show that $\Autocorr$ is Lipschitz continuous so that a small $\eps$ implies an accurate autocorrelation (and hence an accurate power spectrum).

\begin{lemma}
\label{lemma:autocorr_thm}
Let $\covmat$ and $\widehat{\covmat}$ both be symmetric $T \times T$ matrices. Then 
\begin{align}
\new{\|\Autocorr[\covmat] - \Autocorr[\widehat{\covmat}]\|_2} \le    
    \|\Autocorr[\covmat] - \Autocorr[\widehat{\covmat}]\|_1 \,&\le \frac{\sqrt{ 1+\log T }}{N}\|\covmat - \widehat{\covmat}\|_F \label{eq:bound1norm}\\
    \|\Autocorr[\covmat] - \Autocorr[\widehat{\covmat}]\|_\infty &\le \frac{1}{N}\|\covmat - \widehat{\covmat}\|_F \label{eq:boundInfnorm}
\end{align}
where 
$\|\Autocorr[\covmat] - \Autocorr[\widehat{\covmat}]\|_1= \sum_{\tau=0}^{T-1} \big|\autocorr_\tau[\covmat]  - \autocorr_\tau[\widehat{\covmat}] \big|$, 
$\|\Autocorr[\covmat] - \Autocorr[\widehat{\covmat}]\|_\infty = \max_{\tau = 0, \ldots,T-1} \big|\autocorr_\tau[\covmat]  - \autocorr_\tau[\widehat{\covmat}] \big|$,
\new{ and
$\|\Autocorr[\covmat] - \Autocorr[\widehat{\covmat}]\|_2 = \sqrt{ \sum_{\tau = 0}^{T-1} \big|\autocorr_\tau[\covmat]  - \autocorr_\tau[\widehat{\covmat}] \big|^2}$.}
\end{lemma}

\begin{proof} 
\newcommand{\D}{\mathbf{\Delta}}

Define the difference between true covariance matrix and the estimate as $\D = \covmat - \widehat{\covmat}$.
For the $\infty$-norm case \new{in Eq.~\eqref{eq:boundInfnorm}}, using linearity of $\Autocorr$,
\begin{align*}
     \|\Autocorr[\D]\|_\infty &= 
      \max_\tau \|\autocorr_\tau[\D]\|  
      = \frac{1}{N}\max_{\tau}\bigg|\frac{1}{T-\tau}\sum_{t=1}^{T-\tau}\Delta_{t,t+\tau}\bigg| \le \frac{1}{N}\max_{t,t'} |\Delta_{t,t'}| \le \frac{1}{N}\|\D\|_F.
\end{align*}

From this, we immediately have the bound $\|\Autocorr[ \D ]\|_1 \le \frac{T}{N}\|\D\|_F$, but this is loose, and we show below how to derive a better dependence on $T$:
\begin{align}
\big\|&\Autocorr[\covmat] - \Autocorr[\widehat{\covmat}]\big\|_1
= \sum_{\tau=0}^{T-1} \big|\autocorr_\tau [\Delta] \big| 
\le \frac{1}{N}\sum_{\tau=0}^{T-1}\frac{1}{T-\tau}\sum_{t=1}^{T-\tau} |\Delta_{t,t+\tau}| \nonumber \\
&\stackrel{\circled{1}}{\le} \frac{1}{N}\sum_{\tau=0}^{T-1} \sqrt{ \frac{1}{T-\tau} \sum_{t=1}^{T-\tau} |\Delta_{t,t+\tau}|^2 } 
\stackrel{\circled{2}}{\le} \frac{1}{N} \sqrt{\sum_{\tau=0}^{T-1}\frac{1}{T-\tau}} \sqrt{\sum_{\tau=0}^{T-1} \sum_{t=1}^{T-\tau} |\Delta_{t,t+\tau}|^2 } \nonumber \\
&= \frac{1}{N}\sqrt{\sum_{\tau=1}^{T}\frac{1}{\tau}} \sqrt{ \|\Delta\|_F^2 - \sum_{\substack{\alpha\in \textrm{lower triang.}\\ \textrm{off-diag elems}}} \Delta_\alpha^2 }  
\le \frac{\sqrt{ 1+\log T }}{N} \|\Delta\|_F,
\label{ineqn::distr-free-bdd}  
\end{align}
where $\circled{1}$ is due to Jensen's inequality, 
and $\circled{2}$ is due to Cauchy-Schwarz. 

\new{The first inequality in Eq.~\eqref{eq:bound1norm} follows from a general property of the $\|\cdot\|_2$ and $\|\cdot\|_1$ norms}.

\end{proof}

\section{Theoretical Guarantees}
We give bounds on the error of autocorrelation evaluation due to sketching the rows of $\X$, i.e., $\widehat{\X}^\trans = \sketch \X^\trans$. Each row consists of the data at a given time $t$, so this can be trivially implemented in a streaming fashion.  The overall compression ratio is $\gamma=\frac{m}{N}$, \new{which for a fixed $m$ is} 
independent of $T$.

\begin{proposition}
\label{thm::autocorr_thm}
For any $\eps > 0$, and for a data matrix $\X\in\R^{T\times N}$, compute $\widehat{\X} = \X\sketch^\trans \in \R^{T\times m}$ for 
\new{a sketch $\sketch$ with enough rows $m$ such that} 
$\sketch \in \text{JLT}(\eps,\delta,2T)$, and define $\covmat = \X\X^\trans$ and $\widehat{\covmat}=\widehat{\X}\widehat{\X}^\trans$.
Then with probability at least $1-\delta$, the computed autocorrelation based solely on the data sketch satisfies the following error characterizations:
\begin{align}
    \label{eqn::autocorr_error_bd}
\new{ \frac{\|\Autocorr[\widehat{\covmat}] - \Autocorr[\covmat]\|_2}{\|\X\|_F^2} \le }    
    \frac{\|\Autocorr[\widehat{\covmat}] - \Autocorr[\covmat]\|_1}{\|\X\|_F^2} &\le \frac{\sqrt{ 1+\log T }}{N}\varepsilon\\
    \frac{\|\Autocorr[\widehat{\covmat}] - \Autocorr[\covmat]\|_\infty}{\|\X\|_F^2} &\le \frac{1}{N}\varepsilon.
\end{align}
In particular, if $\sketch$ is a Gaussian, Haar or FJLT sketch, then $\sketch \in \text{JLT}(\epsilon,\delta,2T)$ if $m$ is chosen as in Table~\ref{table::JLT}.
\end{proposition}

\begin{proof}
A standard sketching result due to Sarl\'{o}s~\cite{Sarlos_Bounds} gives the error bound for using JLT to estimate  matrix products as the following:
let $\X\in\R^{T_1\times N}$ and $\Y\in \R^{N\times T_2}$. If $\sketch$ is a JLT($\varepsilon,\delta,T_1+T_2$), then \[
\P(\|\X\Y - \X\mathbf{\Omega}^\trans\mathbf{\Omega}\Y\|_F \le \varepsilon\|\X\|_F\|\Y\|_F)\ge 1-\delta
\]
Applying Lemma~\ref{lemma:autocorr_thm} with $\Y=\X$ gives the result immediately.
\end{proof}

To quantitatively characterize how the error in autocorrelation evaluation depends on the compression ratio, we have the following corollary. 
\begin{corollary}
\label{prop::gamma}
Under the setting of Theorem~\ref{thm::autocorr_thm}, assuming the data matrix $\X$ has bounded entries, then the required compression ratio $\gamma = m/N$ to have $ \|\Autocorr[\widehat{\covmat}] - \Autocorr[\covmat]\|_1 \le \varepsilon$ with probability greater than $1-\delta$ is 
$ \gamma = \mathcal{O}\big( \frac{T^2 \log T \log (T/\delta)}{\varepsilon^2 N} \big)$ for Gaussian or Haar matrix sketches, and \remove{$\gamma =  \mathcal{O}\big( \frac{ T^2 \log T \log (2T/(\delta-\e^{-\log^4N}))}{\varepsilon^2 N} \big)$} \new{$\gamma = \mathcal{O}\big( \frac{T^2\log T \log(Nd/\delta)\log(d/\delta)}{\eps^2 N} \big)$} for FJLT sketches.
\end{corollary}

\begin{newstuff}
\begin{proof}
For Gaussian or Haar matrix sketches as a JLT($\widetilde{\eps},\delta,2T$), recall from Table \ref{table::JLT} that the required compressed dimension $m = \mathcal{O}(\widetilde{\eps}^{-2}\log(T/\delta))$. 
Then with probability greater than $1-\delta$, $ \|\Autocorr[\widehat{\covmat}] - \Autocorr[\covmat]\|_1  \le \frac{\sqrt{1+\log T}}{N}\widetilde{\eps}\|\X\|_F^2$
using the 
error characterization equation \eqref{eqn::autocorr_error_bd} in Theorem~\ref{thm::autocorr_thm}. Then, to ensure this $\ell_1$ norm loss bound is less than some $\varepsilon$, the required compression ratio is $\gamma = m/N = \mathcal{O}(\widetilde{\eps}^{-2}\log(T/\delta))/N = \mathcal{O}\big( \frac{T^2 \log T \log (T/\delta)}{\varepsilon^2 N} \big)$, where the last equality exploits $\|\X\|_F^2 = \mathcal{O}(TN)$ since $\X$ has bounded entries.
Similar arguments will give the order of the compression ratio $\gamma$ for FJLT sketches.
\end{proof}
\end{newstuff}

The corollary suggests that as the simulation time $T \rightarrow \infty$, our compression ratio grows, until at some point it is not useful. However, $T$ should be seen as inversely proportional to the lowest desired frequency in the power spectrum, not total simulation time. For longer simulation times $\longT$, the data should be blocked into $B$ matrices $\X_{(1)}, \ldots,\X_{(B)}$, each of size $T = \longT/B$, and then form $\covmat = \frac{1}{B}\sum_{b=1}^B \X_{(b)}\X_{(b)}^\trans$, and similarly for $\widehat{\covmat}$, with fresh sketches $\sketch_{(b)}$ drawn for each block.  If for some reason one needed arbitrarily low frequencies, and wanted the sample time autocorrelation to converge to the true time autocorrelation, then choose $B\propto \sqrt{\longT}$~\cite{timeSeries_Brockwell,Proakis2006}, but otherwise choose $B \propto \longT$ and hence the block size $T$ is constant.

Thus given a fixed time $T$, the corollary says that $\gamma \approx \mathcal{O}(1/N)$ and hence as the amount of data increases, the compression savings are great; in fact, the absolute number of measurements $m$ is independent of the spatial size $N$ \new{for Gaussian and Haar sketches, and only logarithmically dependent on $N$ for the FJLT sketch}. For example, this means that if one increases the resolution of a grid or mesh, the amount of data needed to be stored \new{using a Gaussian sketch} actually stays constant. This holds not just for 1D grids, but 3D or any dimension grids.

We also note that the matrix $\covmat$ need not represent all grid points or particles, but could instead represent a subset of grid points or particles, and then the calculations are done independently for each $\covmat$ and averaged in the end. This may be beneficial in parallel and distributed computing, where each $\covmat$ might represent just the spatial locations stored in local memory.

\begin{newstuff}
\begin{remark}[Error for the power spectral density]
Any bound on $ \|\Autocorr[\widehat{\covmat}] - \Autocorr[\covmat]\|_2 $
immediately translates to a bound on the error of the discrete power spectral density in the Euclidean norm, since the discrete power spectral density is the discrete Fourier transform (DFT) of autocorrelation, and the DFT operator is unitary.
\end{remark}
\end{newstuff}

\section{Numerical Experiments}

The pseudo-code for the proposed sketching algorithm is in Algo.~\ref{alg:sketching}. It exploits existing fast implementations of sample autocorrelation, e.g., \texttt{xcorr} in Matlab or \texttt{numpy.correlate} in Python. 
We use Matlab indexing notation, with $\X(:,j)$ meaning the $j^\text{th}$ column of $\X$, and $\X(i,:)$ the $i^\text{th}$ row.  For our data, the mean was near zero and was not subtracted explicitly. Bartlett windowing~\cite{Proakis2006} was performed to reduce spectral leakage whenever $B>1$.

\begin{algorithm}[ht]
\caption{Sketching for autocorrelation and power density estimation.
Requires existing implementation of \texttt{autocorr}.
}
\label{alg:sketching}
\begin{algorithmic}[1]
\Require Simulation time $\longT$, number of blocks $B$, compression size $m$
\State $T = \longT/B$
\For{$b=0,1,2,\ldots,B-1$}
\State Draw $\sketch \in \R^{m \times N}$ \Comment{One of the \remove{sketches} \new{sketching operators} from \S\ref{sec:sketching}} 
\State Initialize empty array $\widehat{\X}\in \R^{T \times m}$
\For{$t=1,2,\ldots,T$}
    \State Generate data $\bx^\trans \in \R^{1\times N}$
    according to simulation (at time $t+bB$); equivalent to 
    {\it row} $\X(t,:)$ 
    \State Compute and store row $\widehat{\X}(t,:) = (\sketch\bx)^\trans$
    \State Discard $\bx$ from memory
\EndFor
\State Compute $\Autocorr_{(b)} = \frac{1}{N}\sum_{i=1}^m \texttt{autocorr}(\widehat{\X}(:,i))$
\EndFor
\State $\Autocorr = \frac{1}{B}\sum_{b=0}^{B-1} \Autocorr_{(b)}$ \Comment{autocorrelation}
\State $S=\texttt{FFT}(\Autocorr)$ \Comment{power spectral density}
\end{algorithmic}
\end{algorithm}

\begin{remark} \label{rmk:implicit}
Conceptually, the algorithm forms $\widehat{\X} = \X\sketch$, though 
the full-size data matrix $\X$ is never actually formed, as $\widehat{\X}$ is built up row-by-row (and old rows of $\X$ are discarded). Similarly, the estimated covariance matrix $\widehat{\covmat}$, which is introduced for discussion on theoretical properties of sketching methods, is never explicitly constructed for computation, as discussed in Section~\ref{sec:autocorr}.
\end{remark}

\subsection{Baseline methods}
Many existing algorithms for computing autocorrelation require complete data, such as the utility routines provided with the popular MD simulator \texttt{LAMMPS}~\cite{LAMMPS}, so we do not compare with these since they work with the full data. 
Among subsampling approaches, we compare with the following three types of subsampling (recall the data matrix is structured as $\X\in\R^{T\times N}$, where $T$ is the total length of time and $N$ is the total number of particles or grid size), all of which sample with replacement:
\newcommand{\zi}{z_{\tau}^{\mathcal{I}}}
\begin{description}
    \item[Time dimension compression] Given a compression ratio $\gamma$,  sample time points $\mathcal{I}\subset\{1,\ldots,T\}$ with size $|\mathcal{I}|=\lceil\gamma T\rceil$  (where $\lceil a \rceil$ rounds $a$ up to the nearest integer) by  selecting \textbf{rows} from the data matrix $\X$. The natural unbiased estimator for the  autocorrelation  $\autocorr_\tau[{\X}]$  is 
    \begin{equation}\label{eq:timeSubsample}
    \frac{1}{N}\frac{1}{\zi}
    \sum_{t \mid t,t+\tau\in\mathcal{I}} \sum_{i=1}^N \X(t,i)\X(t+\tau,i)
    \end{equation}
    where $\zi$ is a normalization coefficient that is the number of $t$ such that $t\in\mathcal{I}$ and $t+\tau\in\mathcal{I}$ 
    (for full sampling, this is $\zi=T-\tau$ as in \eqref{eqn::autocorr_defn}).  Efficient computation of this autocorrelation estimate is discussed in Remark~\ref{rmk:normalization}.
    When the index $\mathcal{I}$ is sufficiently small, not all lags $\tau$ will have an estimate, thus making computation of the PSD unclear. In these cases, we interpolate the missing lag values using cubic splines.
    
    There are several common choices for $\mathcal{I}$:
    \begin{enumerate}
        \item Choosing $\mathcal{I}$ (pseudo-)randomly according to the uniform distribution. This is the method we use in the experiments unless otherwise noted, as it has the best performance among these types of methods.
        \item Choosing $\mathcal{I}$ via a power-series sampling scheme that is common in simulation of polar liquids (where $\autocorr_\tau[{\X}]$ is only needed for short lags $\tau$ due to the rapid decorrelation).  
        Given a block length $k$, let $\mathcal{I}_0 = \{ 1, 2, 4, 8, \ldots, 2^k\}$, and then the index set $\mathcal{I}$ is divided into blocks
        $\mathcal{I} = \mathcal{I}_0 \cup \left( 2^k + \mathcal{I}_0 \right) \cup \left( 2^{k+1} + \mathcal{I}_0 \right) \cup \ldots $.
        This scheme is intended to give dense sampling for low lags, and some sampling for higher lags while still allowing for reasonable book-keeping due to its structured nature.  See Fig.~\ref{fig:powerSeries} for a comparison of this scheme with random sampling; it generally underperforms random sampling, so we do not present further comparisons.
        
        \item Sparse ruler sampling.
        As shown in Fig.~\ref{fig:powerSeries}, the power-series scheme does not generate all possible lags.  Sampling schemes that do generate all possible lags (up to some point) are known as {\it rulers}, and rulers with only a few samples are {\it sparse rulers}, and are used in signal processing~\cite{SparseRuler}. One can modify the power-series scheme so that each block $\mathcal{I}_0$ is a sparse ruler (we used Wichmann Rulers). The scheme still underperforms random sampling; see 
        \ref{sec:appendix-baseline}
        for more details.
        
        \item Sampling blocks (Algorithm 8 in \cite{FRENKEL_book}), which gives good estimates of $\autocorr_\tau[{\X}]$ for small $\tau$, but does not attempt to estimate $\autocorr_\tau[{\X}]$ for $\tau$ larger than the block size. This does not perform well and details in left for the supplementary information section 1.A.
        
        \item Hierarchical sampling schemes (Algorithm 9 in \cite{FRENKEL_book}), designed to improve on block sampling by giving a small amount of large lag information. This method is exact for some derived quantities (like diffusion coefficients) but {\it ad-hoc} for estimating the large-lag autocorrelation.  This method has high errors (see \ref{sec:appendix-baseline} for details).
    \end{enumerate}
    These last two methods (4 and 5) are different than all the other baseline methods we discuss as they require ``on-the-fly'' computation to record the estimate of $\autocorr_\tau[{\X}]$ for a subset of the lags $\tau$, and this estimate is then updated. These methods do not simply sample $\X$ and then postprocess. Both method 4 and 5 do not give accurate estimates for large lags, hence we do not present further simulation results with these methods.
    
    \item[Particle dimension compression] Given a compression ratio $\gamma$, randomly sample  particles (or grid points) to form $\mathcal{I}\subset \{1,\ldots,N\}$ with size $|\mathcal{I}|=\lceil\gamma N\rceil$
    by uniformly selecting \textbf{columns} from the data matrix $\X$. The natural unbiased estimator of $\autocorr_\tau[{\X}]$ is then
    $$
    \frac{1}{|\mathcal{I}|}\frac{1}{T-\tau}
    \sum_{t=1}^{T-\tau} \sum_{i\in\mathcal{I}} \X(t,i)\X(t+\tau,i).
    $$
    
    \item[Na\"ive uniform sparsification (both time and particles)] Given a compression ratio $\gamma$, uniformly sample $\lceil\gamma TN\rceil$ \textbf{entries} from $\X$.
    This approach has the same estimator for autocorrelation of lag $\tau$ as the case time dimension compression, except that the sampling set $\mathcal{I}$ and normalization constant now depend on the column $i$. We refer to this as ``na\"ive'' since it uses a uniform distribution, in contrast to complicated weighted sampling schemes like \cite{AKL2013Sampling} used in the sampling literature.  With an appropriate normalization $z_{\tau,i}^\mathcal{I}$, the unbiased estimate of  $\autocorr_\tau[{\X}]$ 
    is
    $$
       \frac{1}{z_{\tau,i}^\mathcal{I}}
    \sum_{i=1}^{N} 
    \sum_{\substack{t,\ \text{such that} \\ (t,i),(t+\tau,i)\in\mathcal{I}}}
    \X(t,i)\X(t+\tau,i).
    $$
    which can be calculated via the above formula or via Remark~\ref{rmk:normalization}.
\end{description}

One can combine time dimension and particle dimension compression (doing time-then-particle, or particle-then-time), but for a given overall compression level, we did not find that this improved accuracy, and therefore do not include it in the results.

\begin{remark} \label{rmk:normalization}
To efficiently compute the estimate of the autocorrelation for any time dimension compression scheme, i.e., Eq.~\ref{eq:timeSubsample}, one can use existing fast autocorrelation functions. Specifically, set the non-sampled entries to zero, so they do not contribute to the sum, and put each column of $\X$ through a standard autocorrelation function and then average the results. 
To find the normalization factor $\zi$, one can create an indicator vector $\boldsymbol{\xi}$ where $\xi_t=1$ if $t\in\mathcal{I}$ and $\xi_t=0$ if $t\not\in\mathcal{I}$ (think of this as a ``book-keeping'' particle that can be stored as an extra particle or grid-point), and then compute the autocorrelation of $\boldsymbol{\xi}$ to get the normalization $\zi$.  Computing the value by hand is possible but tedious and the programming is error-prone, which may be a reason why simple (non-random) time compression schemes have historically been favored.
\end{remark}

\begin{figure}
    \centering
    \includegraphics[width=.55\textwidth]{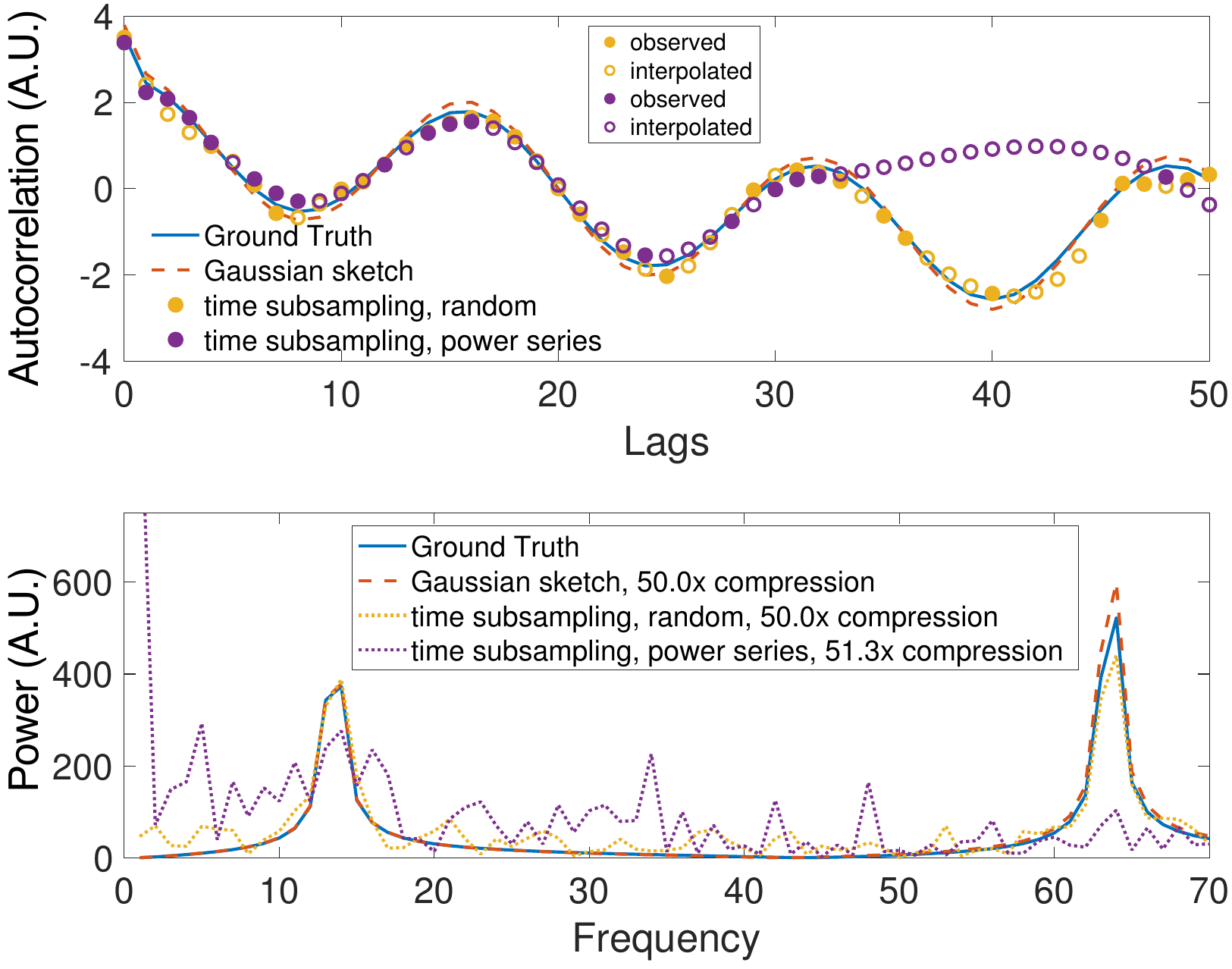}
    \caption{Autocorrelation (top) and power spectral density (bottom) for the two frequency simulation.}
    \label{fig:powerSeries}
\end{figure}

To illustrate the different types of time dimension compression schemes, we conduct a basic experiment of $N=10^4$ particles and $T=2000$ time points with unit spacing, where each particle is randomly assigned one of two possible frequencies (one fast, one slow), and with a random phase; the autocorrelation is the fast sinusoid modulated by the slow sinusoid. The power spectral density ranges up to 500 Hz, of which the first 70 Hz are shown in the bottom of Fig.~\ref{fig:powerSeries}. The ground truth would show two delta functions if $T=\infty$ but are spectrally broadened by the finite time sample.  Fig.~\ref{fig:powerSeries} shows that, at $50\times$ compression, the time sampling approaches have no observations for some lags and must be interpolated. The random time subsampling is more accurate than the power series approach. The Gaussian sketching method \new{requires no interpolation and the PSD it computes} is significantly more accurate. \del{ than both time compression methods.}

\subsection{Methanol ensemble simulation data}
Our dataset is a MD simulation using the \texttt{LAMMPS} software~\cite{LAMMPS}  for $N=384$ methanol molecules with time step 1 fs for 10 ps, with potentials between pairs of bonded atoms, between triplets and between quadruplets of atoms set as harmonic, and potential for pairwise interactions set as the hybrid of the ``DREIDING'' hydrogen bonding Lennard-Jones potential and the Lennard-Jones with cut-off Coulombic potential~\cite{dreiding}. The quantity of interest is the power spectral density of the velocity of the molecules.  Except in Fig.~\ref{fig::TvsError_MD}, no blocking was performed, so $B=1$ and $T=\longT=10000$. The true sample autocorrelation, up to $\tau=100$, is shown in Figure~\ref{fig::autocorr_eg}. The actual simulation was run for $20000$ time steps (20 ps) but the first 10 ps are ignored as the simulation was equilibrating.

Figure~\ref{fig::psd_eg} shows the corresponding true power spectral density (PSD), as well as the PSD computed via the three proposed sketching methods (with Gaussian, Haar and FJLT sketches), as well as the three benchmark methods, using only about $1\%$ of the data. The three sketching methods faithfully recover the true peaks of the spectrum, while the baseline methods (in blue) either have spurious peaks (time compression and naive uniform compression) or miss/distort peaks (particle compression).

\begin{figure}[t]
\centering
\includegraphics[width=.6\columnwidth]{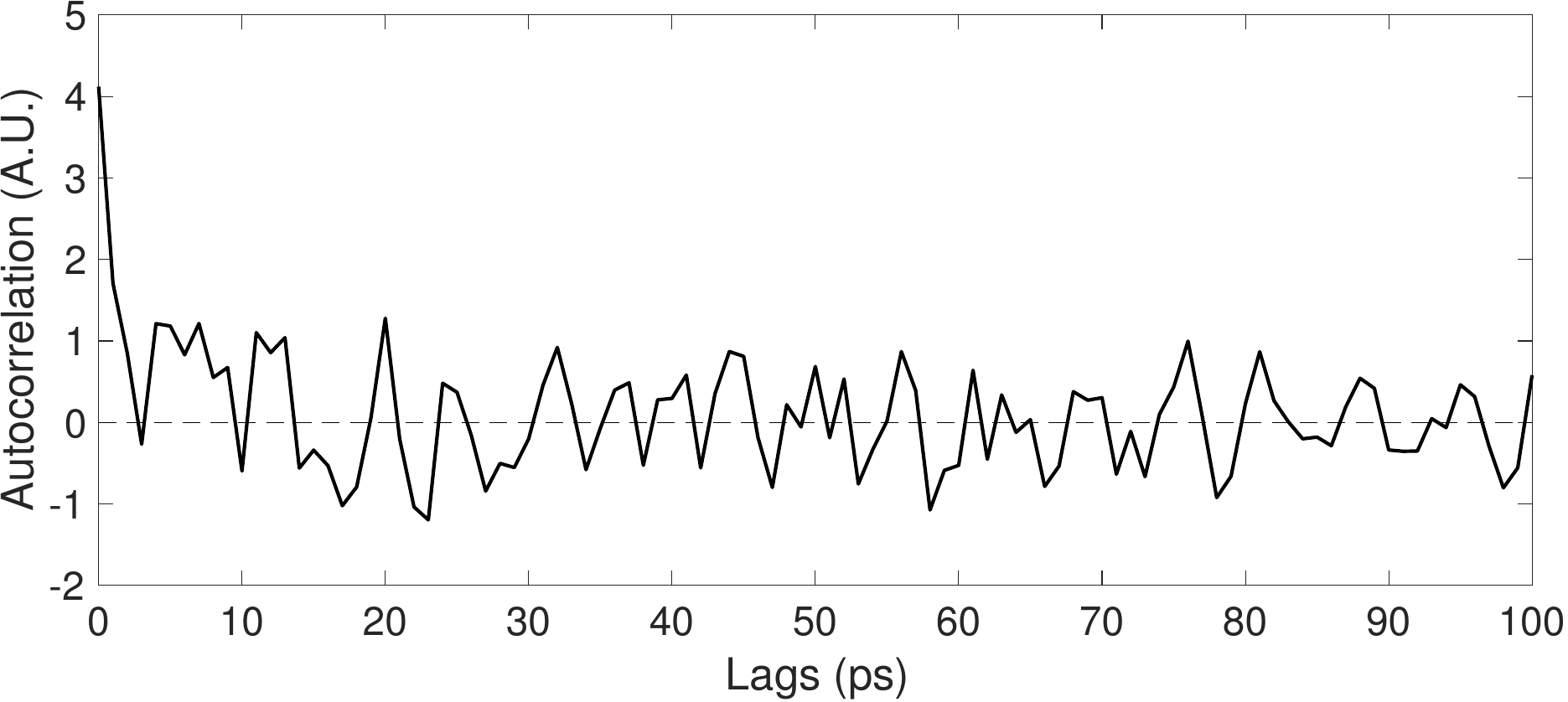}
\caption{Ground truth of autocorrelation of the velocity of methanol molecules up to $\tau=100$.}
\label{fig::autocorr_eg}
\end{figure}

\begin{figure}[t] 
\centering
\includegraphics[width=\columnwidth]{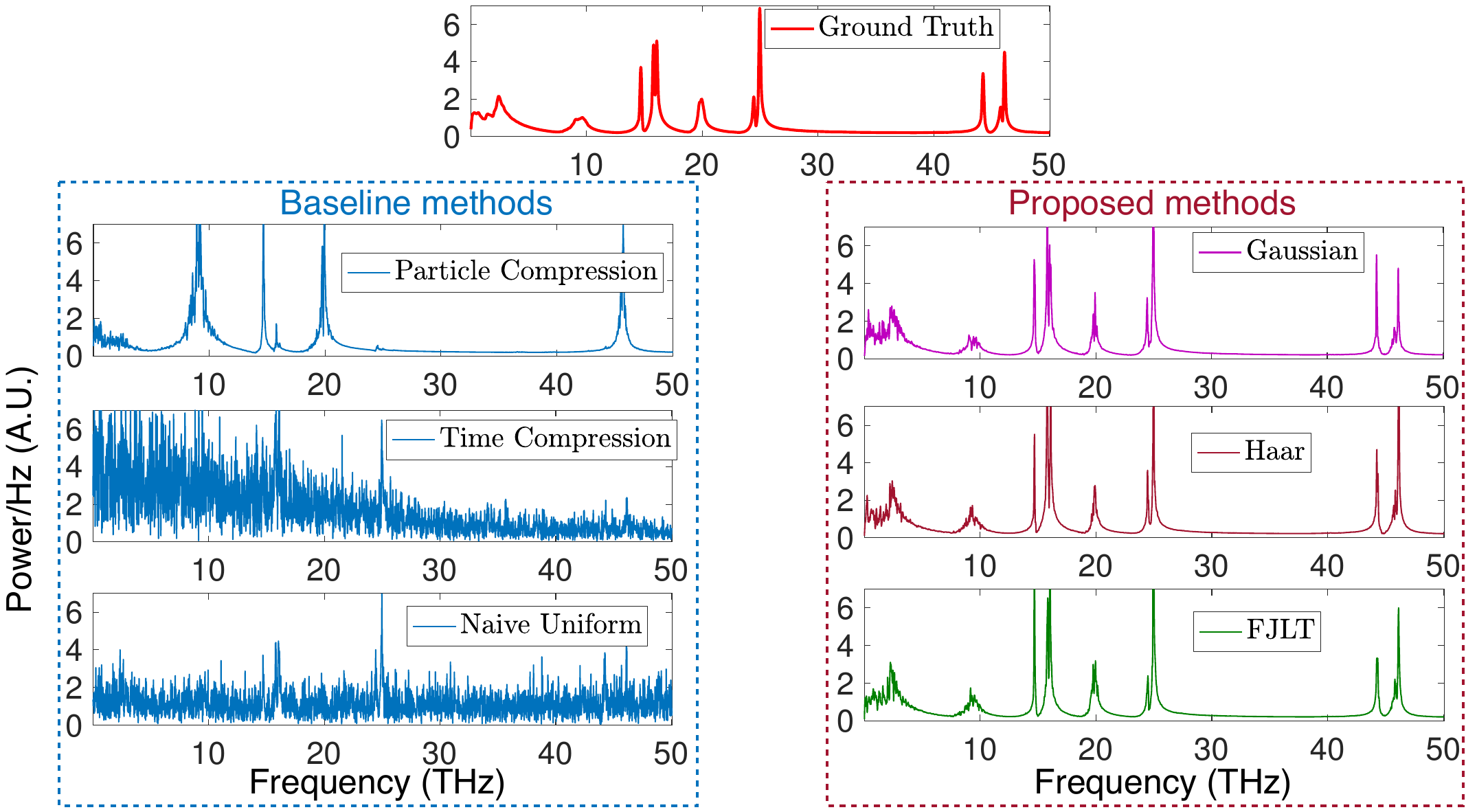}
\caption{Power spectral density for methanol data. 
The compression ratio is 1\% for each method.}
\label{fig::psd_eg}
\end{figure}

For systematic and quantitative comparison, we consider three metrics for evaluating the estimated PSD  $\hat{\mathbf{s}}=\widehat{S}(\omega)$ compared to the true PSD $\mathbf{s}={S}(\omega)$. 
First, we use the relative $\ell_2$ norm
$\|\hat{\mathbf{s}} - \mathbf{s}\|_2/\|\mathbf{s}\|_2 $
which also captures the relative $\ell_2$ error for the autocorrelation (since the Fourier transform is unitary, i.e., Parseval's identity). Second, we use the relative $\ell_\infty$ error, which is defined as 
 $\max_{i, s_i\ne 0} \frac{ | \hat{s}_i - s_i |}{|s_i|}$. 
 Third, we use a relative $\ell_1$ error, defined as $\|\hat{\mathbf{s}} - \mathbf{s}\|_1/\|\mathbf{s}\|_1$, where $\|\mathbf{s}\|_1 = \sum_{i} |s_i|$.

When computing the compression ratio, a sketching method with $\sketch \in \R^{m \times N}$ achieves a $\gamma=m/N$ compression ratio, as no meta-data needs to be stored. The time dimension and particle dimension subsampling methods must also save the time or particle/space indices $\mathcal{I}$ as meta-data, though this is typically insignificant, so they achieve approximately $|\mathcal{I}|/T$ and $|\mathcal{I}|/N$ compression ratios, respectively. The na\"ive uniform sparsification, which samples in both space and time, must save both time and particle/space indices; this is done implicitly by storing the data as a sparse matrix in compressed sparse column format. The overhead of storing these indices can be significant, which is why the compression ratio for ``na\"ive uniform'' is slightly worse than the target of $|\mathcal{I}|/(TN)$.

\begin{figure}[t]
\centering
\includegraphics[width=\columnwidth]{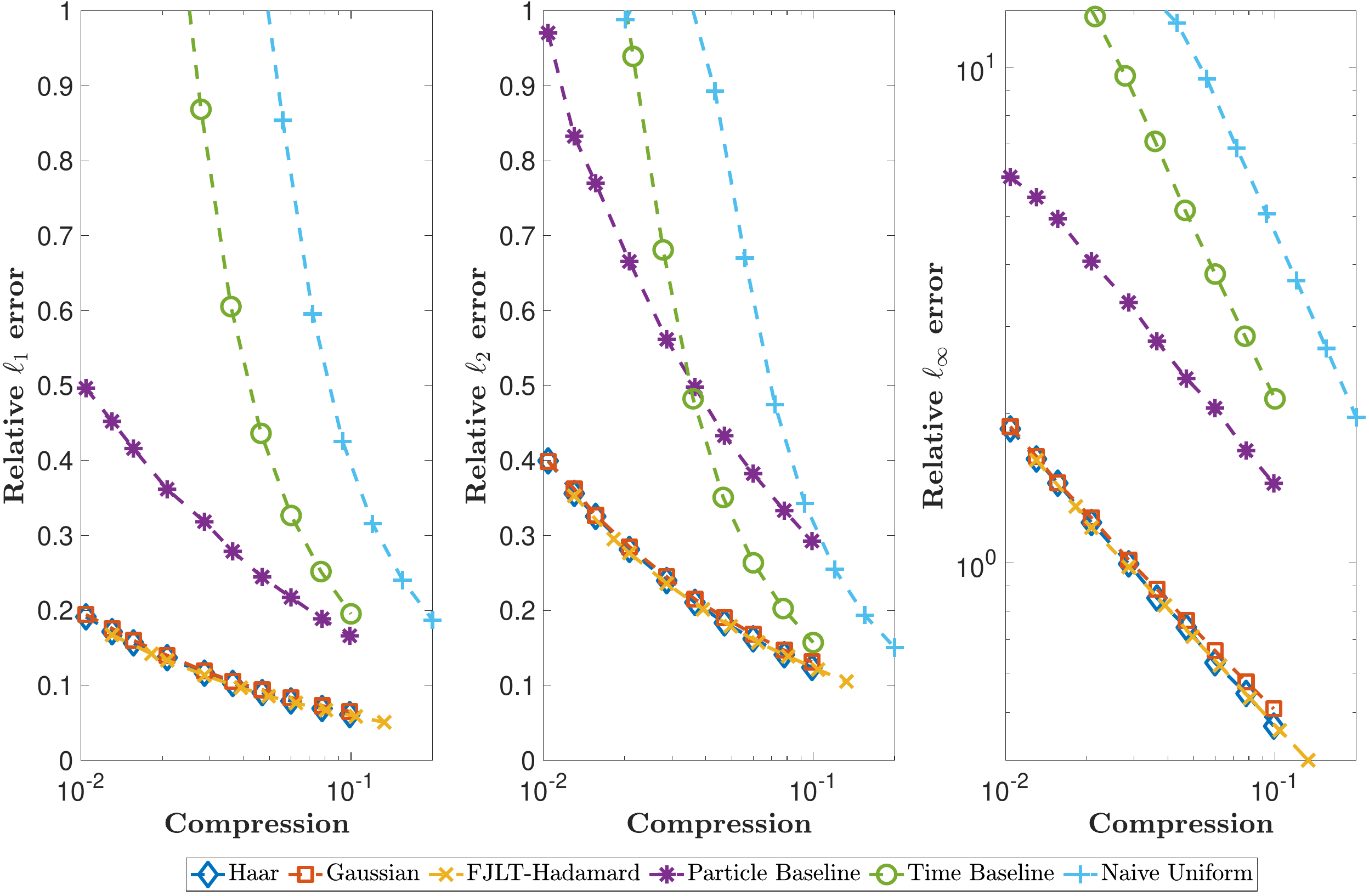}
\caption{
The error due to approximating the PSD for the proposed methods (Haar, Gaussian, and FJLT-Hadamard) compared to baselines, on the methanol data. Left: relative $\ell_1$ error. Middle: relative $\ell_2$ error. Right: relative $\ell_\infty$ error.}
\label{fig::err_multi_MD}
\end{figure}

Figure \ref{fig::err_multi_MD} shows the error metrics as a function of compression ratio $\gamma$ in the interesting regime where $\gamma \ll 1$.
We see that sketching methods perform better than baseline methods in 
the $\ell_1$, $\ell_2$ and $\ell_\infty$ metrics, 
and the advantage is most significant when the compression ratio is small.

\begin{figure}[t]
\centering
\includegraphics[width=\columnwidth]{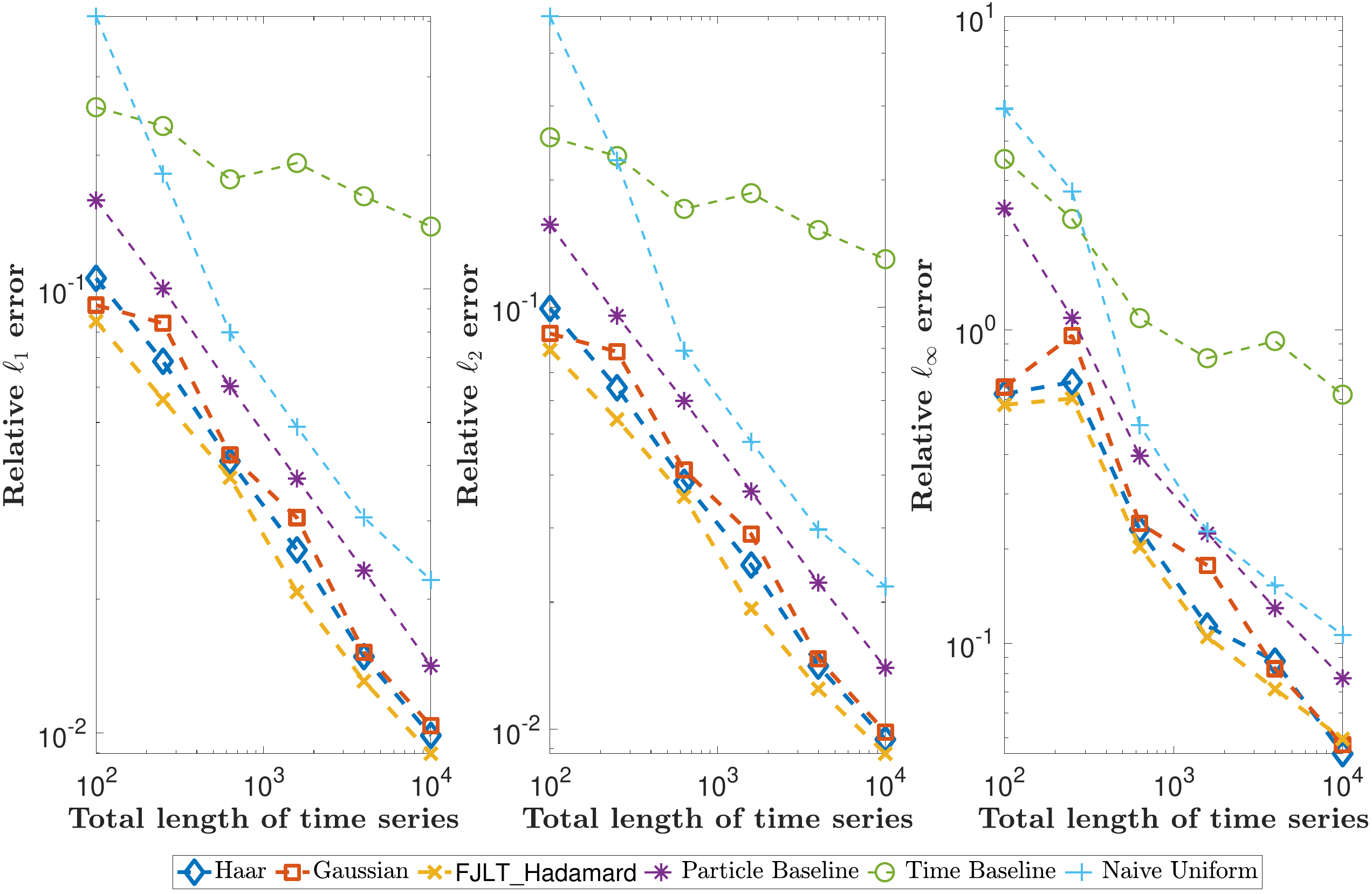}
\caption{Three metrics characterizing the discrepancy between estimated autocorrelation of first 15 lags and the ground truth vs.\ total length of time signals. 
The full time signal is divided into $B=\sqrt{\longT}$ blocks, each of which is used to evaluate the first 15 lags of autocorrelation. }
\label{fig::TvsError_MD}
\end{figure}

Figure \ref{fig::TvsError_MD} shows that the $\ell_1$, $\ell_2$ and $\ell_\infty$ errors 
decay to zero as the time series becomes arbitrarily long. Specifically, we take the total simulation time $\longT \rightarrow \infty$, and set $B=T=\sqrt{\longT}$ (this is necessary, since the simpler choice of $B=1$ and $T=\longT$ does not give a consistent estimator even with fully sampled data). The evaluation of the errors of the autocorrelation are with respect to the first 15 lags. The compression ratio of all sketching methods is fixed as 10\%.  The figure shows that all methods appear to be consistent, with the sketching methods significantly more accurate  
compared to the {\it ad hoc} baselines.

\paragraph{Synthetic data}
The performance of the sketching methods over the classical benchmark methods is significant, but in fact the discrepancy can be arbitrarily large. 
\ref{sec:appendix-synthetic}  shows a synthetic data set created to be adversarial for the classical methods, for which they perform poorly, whereas the sketching methods do well.  The data is created to have a few ``special'' particles which contribute significantly but are unlikely to be sampled by the particle sampling methods, and to have a few short pulses, so that the relevant time dynamics is likely to be missed by the time sampling methods. 
The sketching methods are not susceptible to such adversarial examples.

\section{Conclusions} Since second order statistics like autocorrelation and power density spectral can be computed via the empirical covariance matrix, this means that sketching methods can be used to preserve statistical properties of the data. These sketching methods come with well-understood theory, little extra computational burden, straightforward implementation, and excellent practical performance. For these reasons, we hope they find their place in the numerical simulation toolkit.  An interesting future question is whether even more powerful practical estimators of autocorrelation can be achieved by bypassing the estimation of the covariance matrix.

\section*{Acknowledgments}
The authors thank Michael Wakin for helpful discussions on fast computation of autocorrelation, Marc Thomson for providing the molecular dynamics data,  Francis Starr for discussions of sampling schemes for water,
\new{and anonymous reviewers for helpful comments (notably, pointing out a better sample complexity bound for the FJLT)}. This material is based upon work supported by the National Science Foundation under grant no.\ 1819251.

\appendix


\section{Further experiments} \label{sec:appendix}

\subsection{Alternative baseline methods} \label{sec:appendix-baseline}
We expand on other alternatives for time-dimension compression (beyond the (1) random and (2) power-series sampling), namely
\begin{enumerate}\addtocounter{enumi}{2}
        \item Sparse ruler sampling.
        The power-series scheme does not generate all possible lags.  Sampling schemes that do generate all possible lags (up to some point) are known as {\it rulers}, and rulers with only a few samples are {\it sparse rulers}. One can modify the power-series scheme so that each block $\mathcal{I}_0$ is a sparse ruler (we used Wichmann Rulers). 
        
        \item Sampling blocks (Algorithm 8 in \cite{FRENKEL_book}), which gives good estimates of $\autocorr_\tau[{\X}]$ for small $\tau$, but does not attempt to estimate $\autocorr_\tau[{\X}]$ for $\tau$ larger than the block size. 
        
        \item Hierarchical sampling schemes (Algorithm 9 in \cite{FRENKEL_book}), designed to improve on block sampling by giving a small amount of large lag information. This method is exact for some derived quantities (like diffusion coefficients) but {\it ad-hoc} for estimating the large-lag autocorrelation.  This method has high errors. 
\end{enumerate}

\begin{figure}
    \centering
    \includegraphics[width=.65\textwidth]{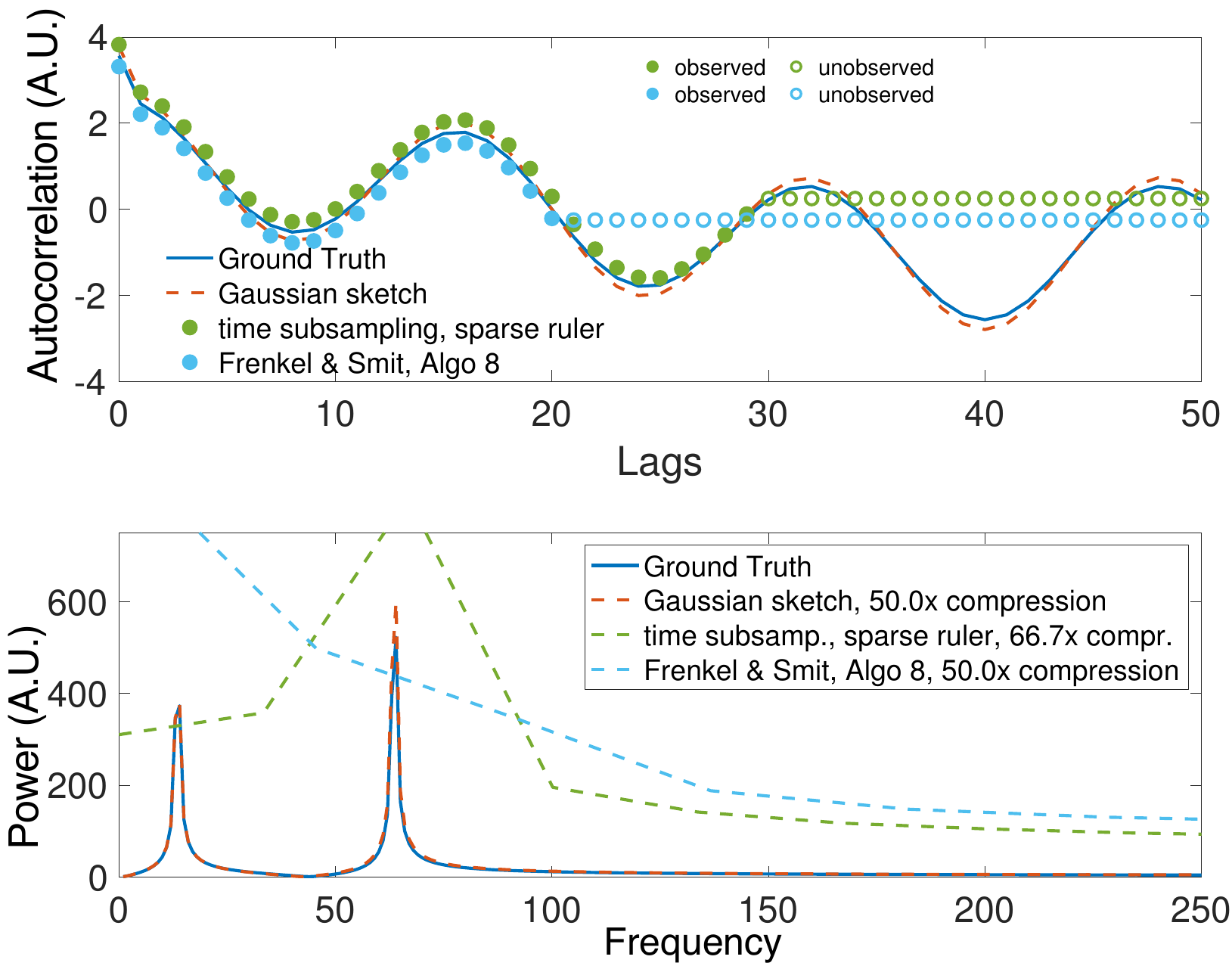} 
    \caption{Top: autocorrelation, and bottom: Power spectral density (PSD) for a synthetic simulation.  The sparse ruler subsampling and the block (Algorithm 8) subsampling miss sampling the autocorrelation at long lags, with the effect of making the PSD estimate have low resolution.  Y-axis in arbitrary units for both plots.
    }
    \label{fig:algo8}
\end{figure}

Fig.~\ref{fig:algo8} compares the sparse ruler sampling and block sampling (Algorithm 8), as well as using the Gaussian sketch. This uses the same $N=10000$ and $T=2000$ synthetic data as in Figure 1 in the main text.  Both the sparse ruler sampling and block sampling only observe the autocorrelation for short lags. For this reason, the autocorrelation cannot even be interpolated at missing lags, but rather these values must be extrapolated. Rather than do this, the PSD is computed using only the short time lags, but this has the effect of lowering the resolution of the PSD.  The bottom part of the figure shows the PSD.

Fig.~\ref{fig:algo9} demonstrates the hierarchical sampling scheme on the same data. This scheme samples in blocks (giving a good estimate of short-time autocorrelation lags, much like the block sampling scheme), but then also aggregates blocks to estimate longer lag autocorrelation.  For some quantities, such as the diffusion constant when defined as the integral of autocorrelation (e.g., in the discrete case, this is just a sum), this aggregation-by-averaging results in no loss. However, for estimating the autocorrelation itself, the estimate is highly inaccurate.  The corresponding PSD is not shown as it is considerably inaccurate.

\begin{figure}
    \centering
    \includegraphics[width=.65\textwidth]{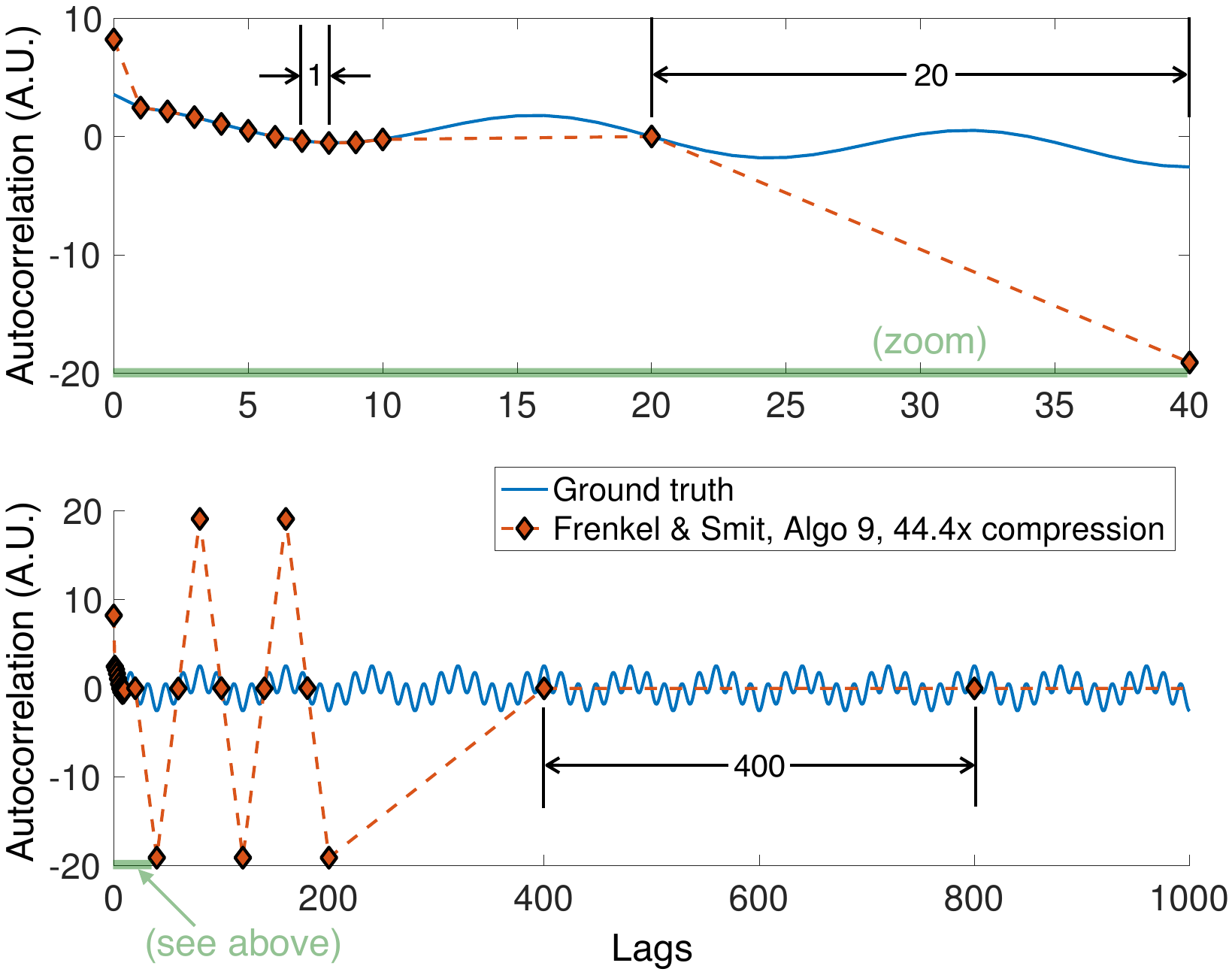} 
    \caption{Autocorrelation, demonstrating the hierarchical sampling scheme of Algorithm 9.  The top plot is a zoomed in version of the bottom plot.  The estimate of the autocorrelation at long lags is inaccurate, and the resulting PSD is unusable. } 
    \label{fig:algo9}
\end{figure}

\subsection{Synthetic data} \label{sec:appendix-synthetic}
The main paper presents realistic data and shows that newly proposed sketching methods outperform classical methods.  Here, we show that the difference in performance can be made almost arbitrarily large by choosing adversarial synthetic data.  The specific random nature of the sketching methods makes it impossible to create generic adversarial examples, whereas the classical methods which rely on weaker notions of randomness are much more susceptible.

\paragraph{Creation of the data set}
Consider a collection of $N=10,\!000$ particles among which 9997 of them share the same eigenfrequency $\omega$ while 3 particles have an additional eigenfrequency $\omega'$. The existence of special particles contributes to the inhomogeneity of the ensemble dynamics. Furthermore, there are 2 pulses in the time range for every particle in the ensemble. Each pulse can be represented by 
$p_1(t)=p(t-t_1)$, $p_2(t) = p(t-t_2)$ and 
$p(t) = 10\sin\left(\frac{\pi}{\delta}t\right)\mathbbm{1}(- \frac{\delta}{2} \le t \le \frac{\delta}{2})$, where 
 $\delta\approx 0.6 \cdot \frac{2\pi}{\omega}$ which accounts for more than half of a period of the signal with common eigenfrequency, and $\mathbbm{1}$ is the $0$-$1$ indicator function. Each particle has a random phase $\varphi_i\in[0,2\pi)$.  
Specifically, 9997 particles have the ``common'' dynamics
\[
(i = 1, \ldots, 9997 )\quad x^{\textrm{common}}_i(t) = \sin(\omega t + \varphi_i) + p_1(t) + p_2(t) + \varepsilon_{i}(t)
\]
while 3 ``special'' particles have one more ingredient in their dynamics 
\[
(j = 9998, 9999, 10000 )\quad  x^{\textrm{special}}_j(t) = \sin(\omega t + \varphi_j) + 80\sin(\omega' t + \varphi'_j) +  p_1(t) + p_2(t) + \varepsilon_{j}(t)
\]
so that when taking the expectation the additional frequency component demonstrate significant importance in the overall spectrum, and $\varepsilon(t)$ is white noise. 
Figure \ref{fig::syn_eg} shows the signal example of a common particle and a special particle, while the ground truth autocorrelation and power spectral density are shown in Figure \ref{fig::syn_autocorr_psd}. 
\begin{figure}[ht]
\centering
\includegraphics[width=.65\textwidth]{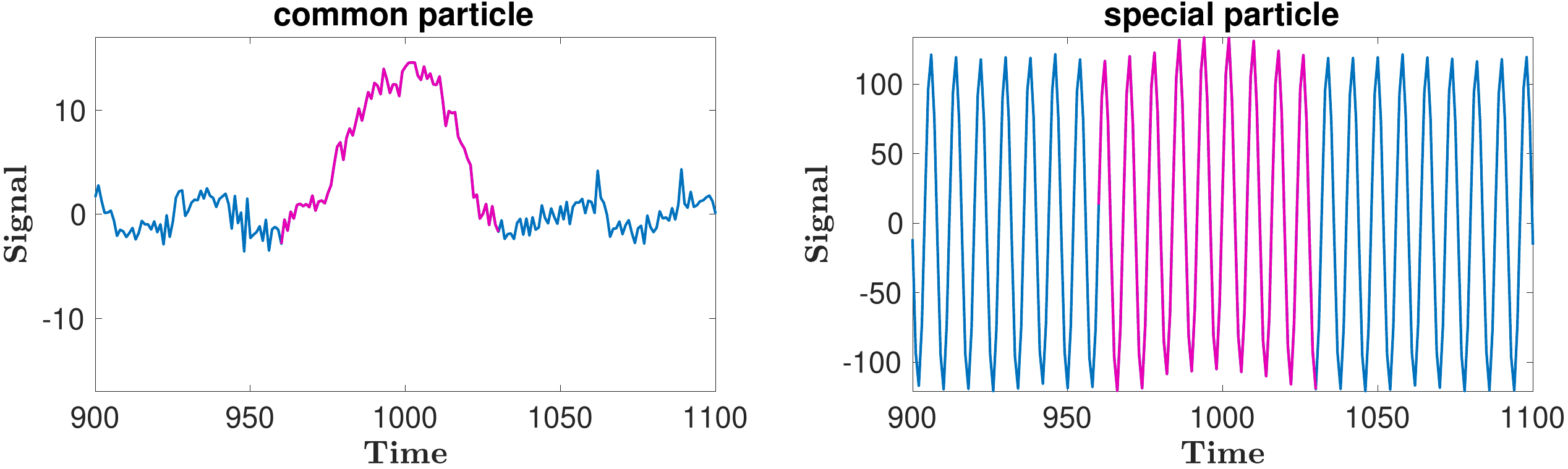}
\caption{Example of particle dynamics in synthetic data. The left subfigures shows the signal of a common particle and the right subfigure shows the signal of a particle with two eigen-frequencies. 2 pulses exist in the synthetic signal and are introduced apart from each other thus not merging their peaks, while we show the zoomed version of one pulse, which is marked in the colour of magenta.}
\label{fig::syn_eg}
\end{figure}

\begin{figure}[ht]
\centering
\includegraphics[width=.65\textwidth]{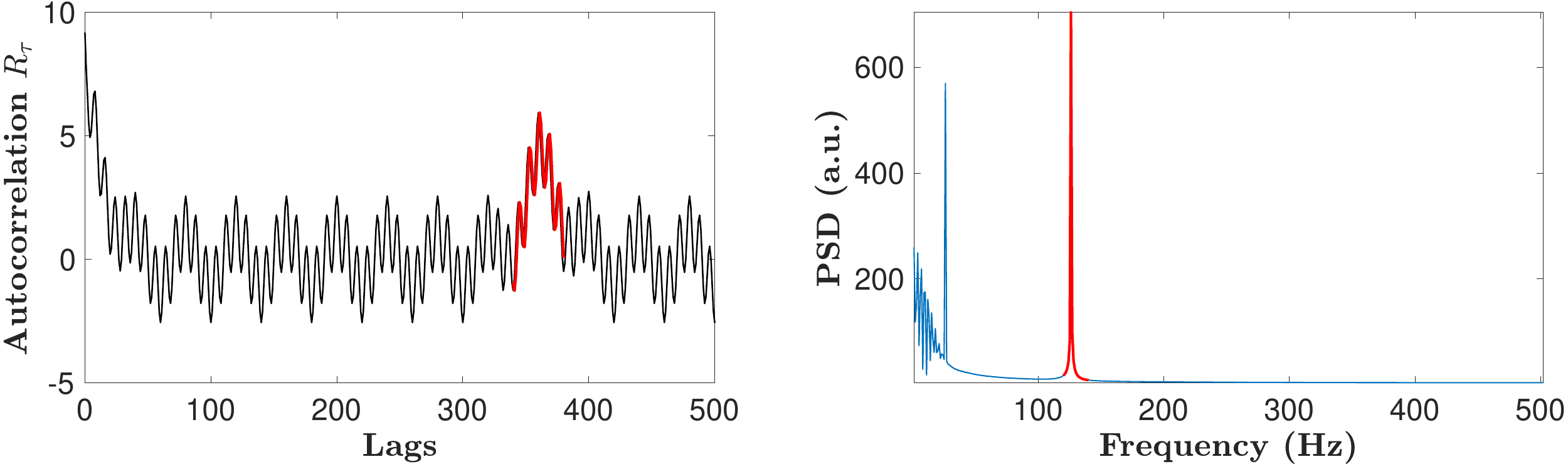}
\caption{Autocorrelation and power spectral density of the synthetic data. The red peak in the power spectral density exists because of special particles, and the red lags in autocorrelation are due to existence of pulses. }
\label{fig::syn_autocorr_psd}
\end{figure}

Figure \ref{fig::synthe_pulse_multi_error} shows the performance of each sketching method on evaluating the power spectral density of the synthetic data set. 
The sketching methods perform well, whereas the classical baseline methods perform so poorly as to be unusable.  For the sketching methods, even when compression is around 1\%, the characteristic peak in the PSD formed by the 3 special particles is still correctly identified, whereas it is completely missed by all 3 classical methods. This is mostly demonstrated by the relative $\ell_\infty$ error which captures the largest discrepancy in PSD evaluation at any frequency.  In fact, all the baseline methods have over 100\% relative error on the $\ell_\infty$ error, regardless of compression.

\begin{figure}[ht]
\centering
\includegraphics[width=.65\textwidth]{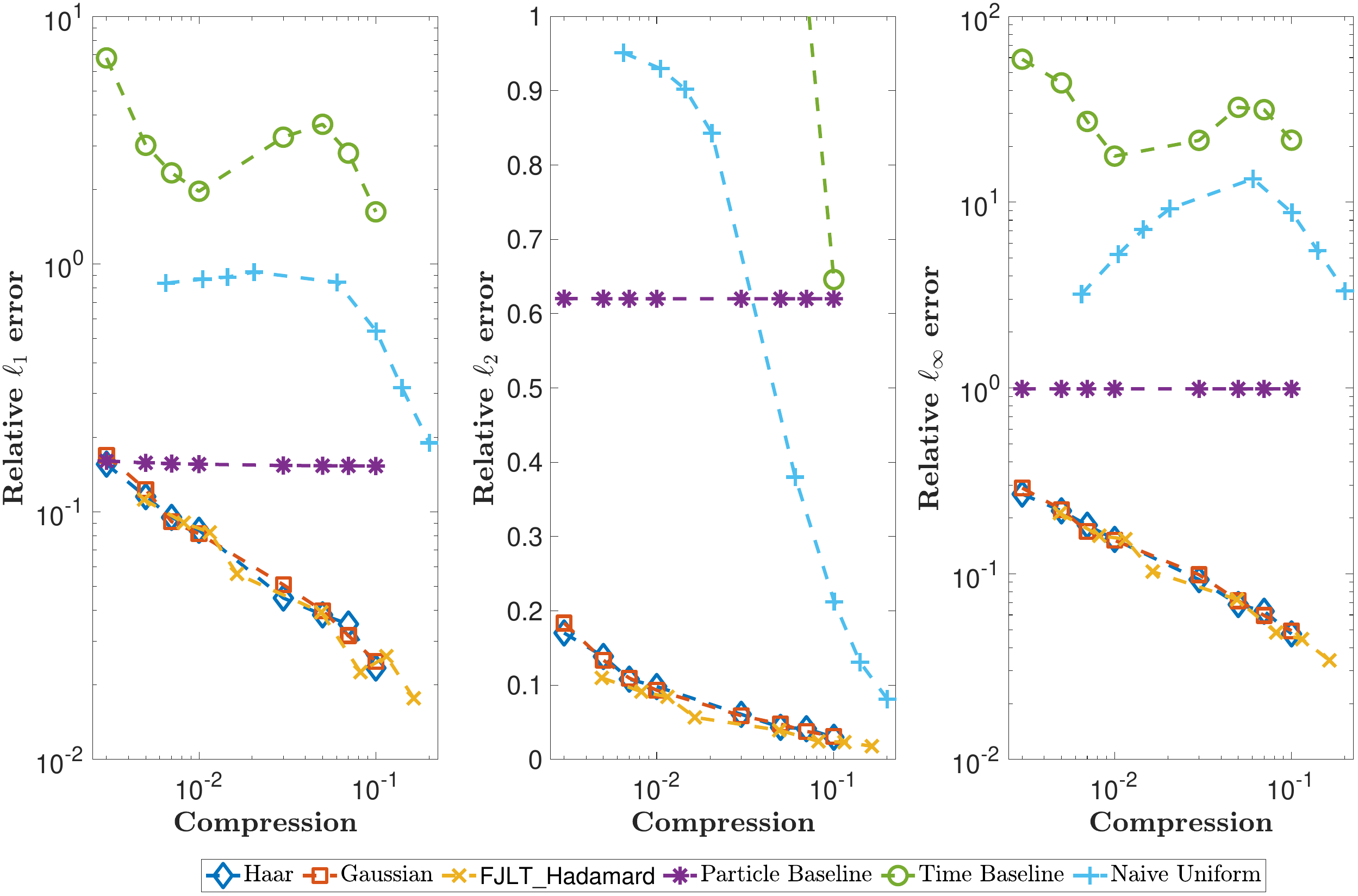}
\caption{Three metrics characterizing accuracy of sketching methods on the PSD in the case of adversarial synthetic data.}
\label{fig::synthe_pulse_multi_error}
\end{figure}

\begin{newstuff}

Figure~\ref{fig::A10_variance} is the same experiment as  Figure \ref{fig::synthe_pulse_multi_error} but also reports information on the variance with respect to the $\ell_1$ errors. Specifically, box plots are shown, with the middle red line showing the median, and the top and bottom of the box are the 75\% and 25\% percentiles, respectively.
The boxes for the sketching approaches appear large, but due to the logarithmic scale of the $y$-axis, there is actually not too much spread.  The time baseline is inaccurate and has large spread; the naive uniform baseline has less spread but is also inaccurate.  The particle baseline shows reasonable good performance for the median, but has worrisome outliers (as indicated by the red $+$ symbols). This is expected for this particular synthetic setup, since the method is reasonable at capturing most of the behavior as long as it does \emph{not} sample one of the three ``special'' particle. In the cases when it does sample a ``special'' particle, the method has no way to know that these particles are rare, so due to the normalization, it heavily weights these particles and incorrectly estimates their effect. These are the outliers shown in the figure, and their effect gets larger as $\gamma\to 0$ since the normalization factor grows.  The variation with respect to the $\ell_\infty$ and $\ell_2$ metrics are similar.

\begin{figure}[ht] 
\centering
\includegraphics[width=.96\textwidth]{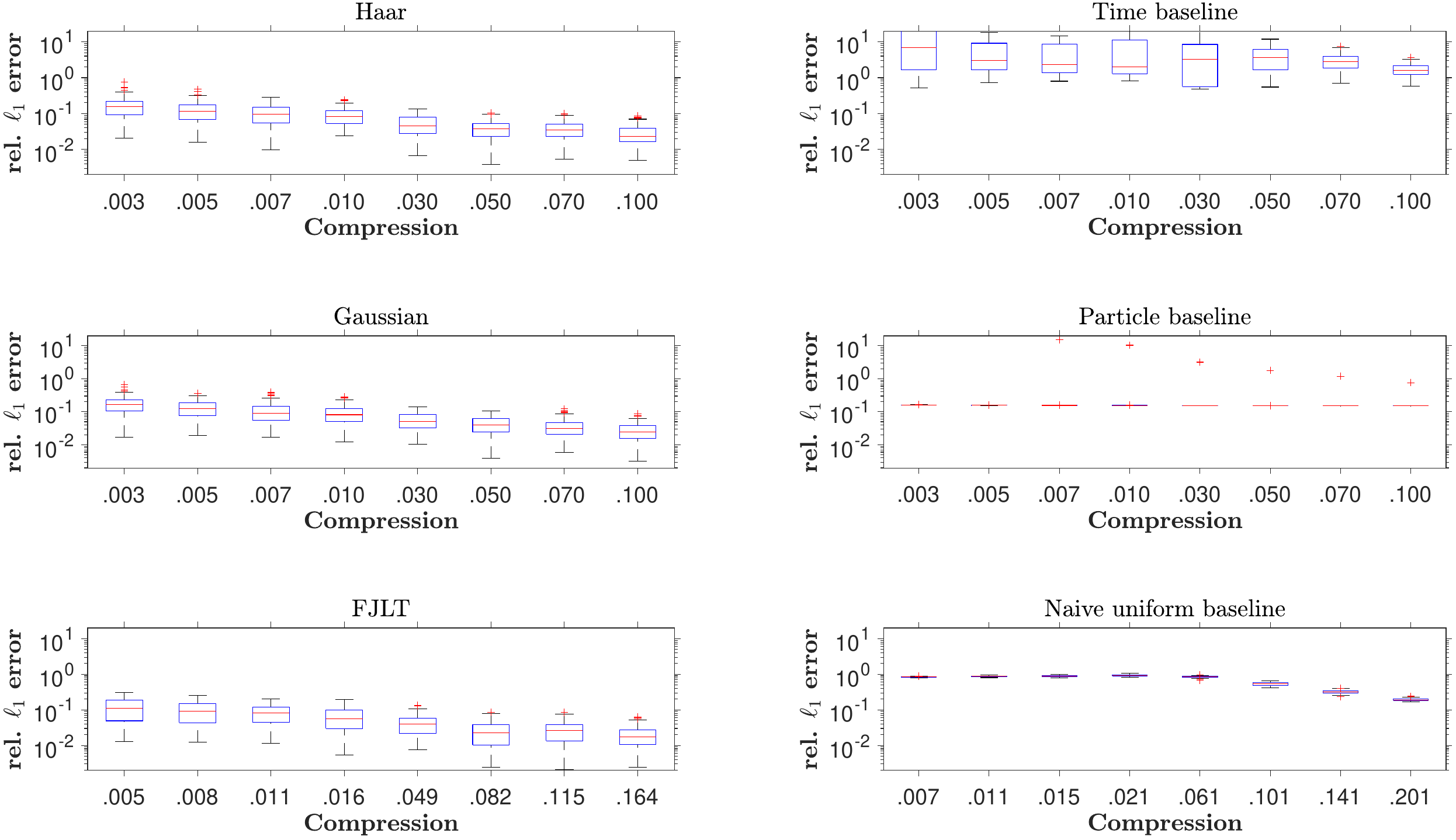}
\caption{\new{Variability of relative $\ell_1$ errors (as reported in Figure \ref{fig::synthe_pulse_multi_error}) due to approximating the PSD for the proposed methods (Haar, Gaussian, and FJLT-Hadamard) compared to baselines, on the methanol data. With respect to each compression mark and each sketching method, the experiment is repeated for 100 trials.}}
\label{fig::A10_variance}
\end{figure}

\end{newstuff}

\subsection{Variance information}
\begin{newstuff}
The following plots show the variation of errors (as reported respectively in Figures \ref{fig::err_multi_MD} and \ref{fig::TvsError_MD}) 
when sketching methods are used to evaluate PSD/autocorrelation. On each box, the central mark indicates the median, and the bottom and top edges of the box indicate the 25th and 75th percentiles, respectively. The whiskers extend to the most extreme data points not considered outliers, and the outliers are plotted individually using the `+' symbol, if any.

We only show data for the relative $\ell_2$ norm errors, but the results for $\ell_1$ and $\ell_\infty$ norm errors are similar.

Fig.~\ref{fig:variance_PSD} shows that for approximating the PSD, the sketching methods have a reasonably small spread. The variance seems to increase as the compression ratio $\gamma \to 0$ which makes sense since there is less averaging when there are fewer samples. The time and naive subsampling baseline methods have reasonably low spread too, but very large errors.  The particle baseline has a large variance in all compression regimes.

Fig.~\ref{fig:variance_autocor} shows the variability when approximating the first 15 lags of the autocorrelation. All methods have somewhat similar variance at a given error level. However, note that the $y$-axis is log scale, so if two boxes seem the same size but one is centered at a lower relative error, then that box represents less spread of the data. Hence we again see the trend that most of the methods have lower variance when there is more data (larger $\gamma$) since they are also more accurate in this regime.
\end{newstuff}

\begin{figure}[ht]
\centering
\includegraphics[width=.96\textwidth]{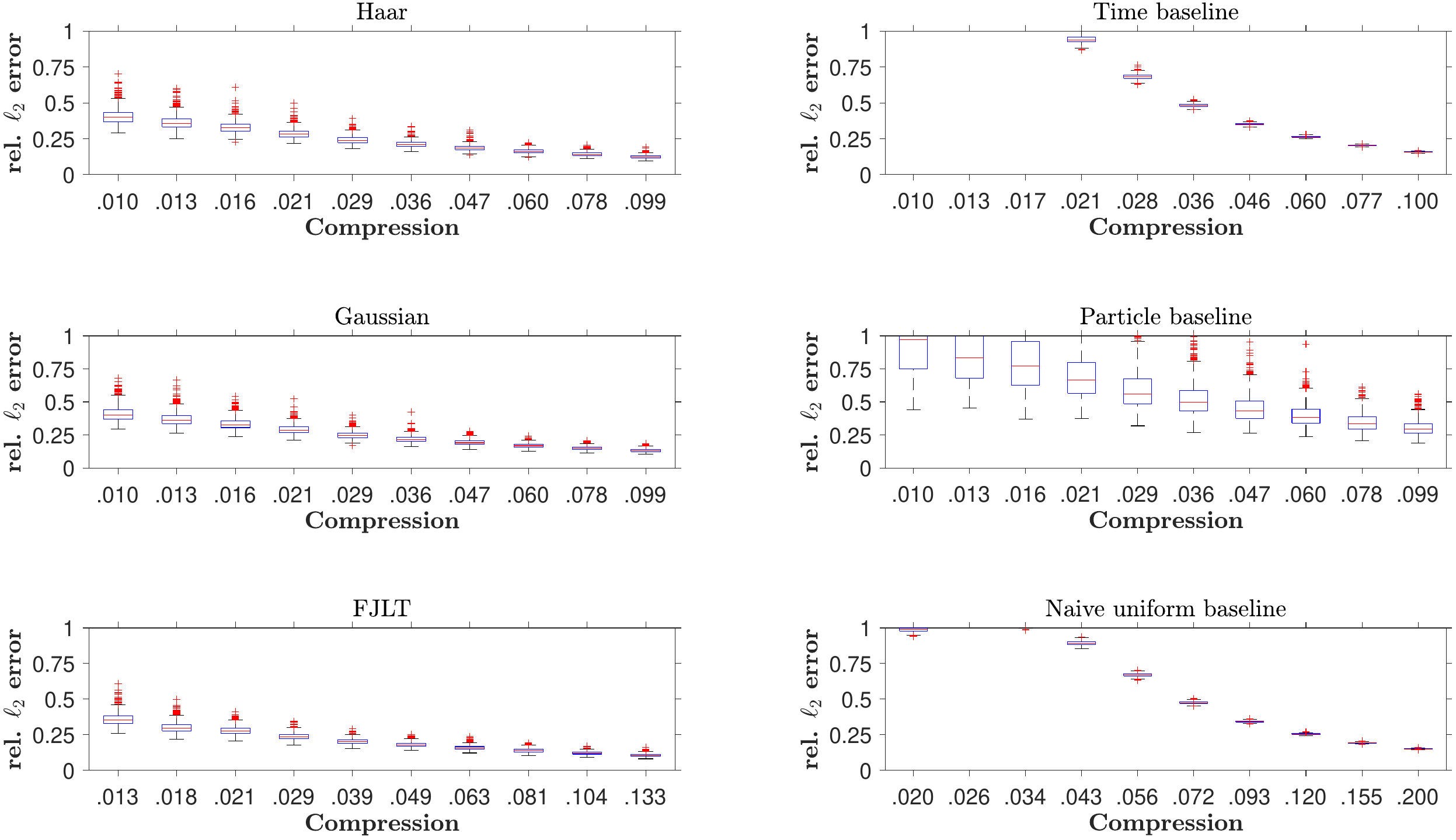}
\caption{\new{Variability of relative $\ell_2$ errors (as reported in Figure \ref{fig::err_multi_MD}) due to approximating the PSD for the proposed methods (Haar, Gaussian, and FJLT-Hadamard) compared to baselines, on the methanol data. With respect to each compression mark and each sketching method, the experiment is repeated for 1000 trials.}}
\label{fig:variance_PSD}
\end{figure}

\begin{figure}[ht]
\centering
\includegraphics[width=.96\textwidth]{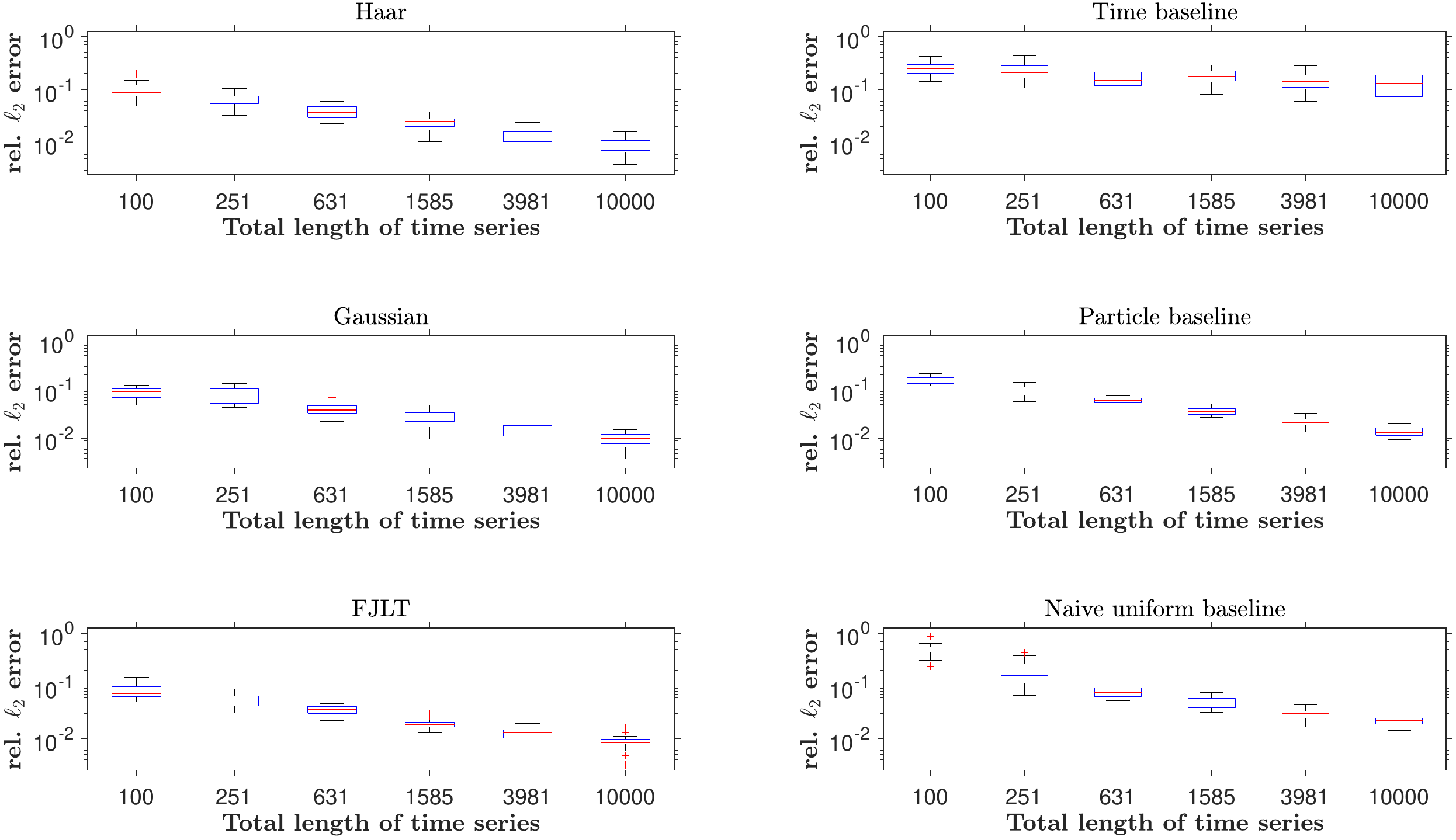}
\caption{\new{Variability of $\ell_2$ error of the estimated autocorrelation of first 15 lags (as reported in Figure \ref{fig::TvsError_MD}) vs.\ total length of time signals. With respect to each fixed length of time series and each sketching method, the experiment is repeated for 20 trials.}}
\label{fig:variance_autocor}
\end{figure}

\end{document}